\documentclass{article}

\usepackage{microtype}
\usepackage{graphicx}
\usepackage{subfigure}
\usepackage{booktabs} 

\usepackage{hyperref}

\usepackage[accepted]{icml2023}

\usepackage{amsmath}
\usepackage{amssymb}
\usepackage{mathtools}
\usepackage{amsthm}
\usepackage{xspace}
\usepackage{paralist}
\usepackage{multirow}
\usepackage{amsfonts}
\usepackage{dsfont}

\usepackage[capitalize,noabbrev]{cleveref}

\theoremstyle{plain}
\newtheorem{theorem}{Theorem}[section]
\newtheorem{proposition}[theorem]{Proposition}

\theoremstyle{definition}
\newtheorem{definition}[theorem]{Definition}

\theoremstyle{remark}

\icmltitlerunning{GraphCleaner}

\begin{document}
\newtheorem{defn}{Definition}[section]
\newtheorem{thm}{Theorem}[section]
\newtheorem{prop}{Proposition}[section]
\newtheorem{problem}{Problem}


\newcommand{\hide}[1]{}
\newcommand{\todo}[1]{{\textsf{\textcolor{blue}{{ [#1]}}}}}
\newcommand{\reminder}[1]{{\textsf{\textcolor{red}{[#1]}}}}
\newcommand{\atn}[1]{\textcolor{red}{#1}}
\newcommand{\notice}[1]{\textcolor{blue}{#1}}
\newcommand{\cmt}[1]{\emph{\textcolor{cyan}{#1}}}



\newcommand{\method}{\textsc{GraphCleaner}\xspace}  
\newcommand{\agroperator}{\textsc{Agr}\xspace}

\newcommand{\predonly}{\textit{ProbOnly}\xspace}
\newcommand{\neionly}{\textit{NeighOnly}\xspace}

\newcommand{\kdtree}{KD-tree\xspace}
\newcommand{\trustscore}{trustworthiness score\xspace}
\newcommand{\jiangtrustscore}{Trust Score\xspace}
\newcommand{\mislabtsk}{mislabel detection task\xspace}
\newcommand{\aggoperator}{\textsc{Agg}\xspace}

\newcommand{\ReturnN}[1]{\State \textbf{return} #1}

\twocolumn[
\icmltitle{GraphCleaner: Detecting Mislabelled Samples in Popular Graph Learning Benchmarks}



\icmlsetsymbol{equal}{*}

\begin{icmlauthorlist}
\icmlauthor{Yuwen Li}{comp}
\icmlauthor{Miao Xiong}{ids}
\icmlauthor{Bryan Hooi}{comp,ids}
\end{icmlauthorlist}

\icmlaffiliation{comp}{School of Computing, National University of Singapore, Singapore}
\icmlaffiliation{ids}{Institute of Data Science, National University of Singapore, Singapore}

\icmlcorrespondingauthor{Yuwen Li}{yuwenli@u.nus.edu}

\icmlkeywords{Machine Learning, ICML}

\vskip 0.3in
]



\printAffiliationsAndNotice{}  

\begin{abstract}

Label errors have been found to be prevalent in popular text, vision, and audio datasets, which heavily influence the safe development and evaluation of machine learning algorithms.
Despite increasing efforts towards improving the quality of generic data types, such as images and texts, the problem of mislabel detection in graph data remains underexplored. 
To bridge the gap, we explore mislabelling issues in popular real-world graph datasets and propose \textsc{GraphCleaner}, a post-hoc method to detect and correct these mislabelled nodes in graph datasets. 
\textsc{GraphCleaner} combines the novel ideas of 1) \emph{Synthetic Mislabel Dataset Generation}, which seeks to generate realistic mislabels; and 2) \emph{Neighborhood-Aware Mislabel Detection}, where neighborhood dependency is exploited in both labels and base classifier predictions.
Empirical evaluations on 6 datasets and 6 experimental settings demonstrate that \textsc{GraphCleaner} outperforms the closest baseline, with an average improvement of $0.14$ in F1 score, and $0.16$ in MCC. 
On real-data case studies, \textsc{GraphCleaner} detects real and previously unknown mislabels in popular graph benchmarks: \texttt{PubMed}, \texttt{Cora}, \texttt{CiteSeer} and \texttt{OGB-arxiv}; we find that at least $6.91\%$ of \texttt{PubMed} data is mislabelled or ambiguous, and simply removing these mislabelled data can boost evaluation performance from 86.71\% to 89.11\%\footnote{Corrected datasets and code are available at https://github.com/lywww/GraphCleaner/tree/master.}.
\end{abstract}

\section{Introduction}
\label{sec:intro}
Data is the primary input to any AI system, both for learning and evaluation. Hence, recognizing \emph{data quality} as a critically important factor for the success of AI systems, \emph{data-centric AI} has rapidly emerged in recent years. The vision of data-centric AI is to ensure clean and high-quality data in all phases of the project life-cycle and to develop tools and processes for monitoring and improving data quality. In particular, in supervised learning contexts, correcting label noise, including mislabelled and ambiguously labelled data, is of high priority due to the key importance of labels in the training and evaluation process.

\begin{figure*}[h]
    \centering
    \includegraphics[width=\linewidth]{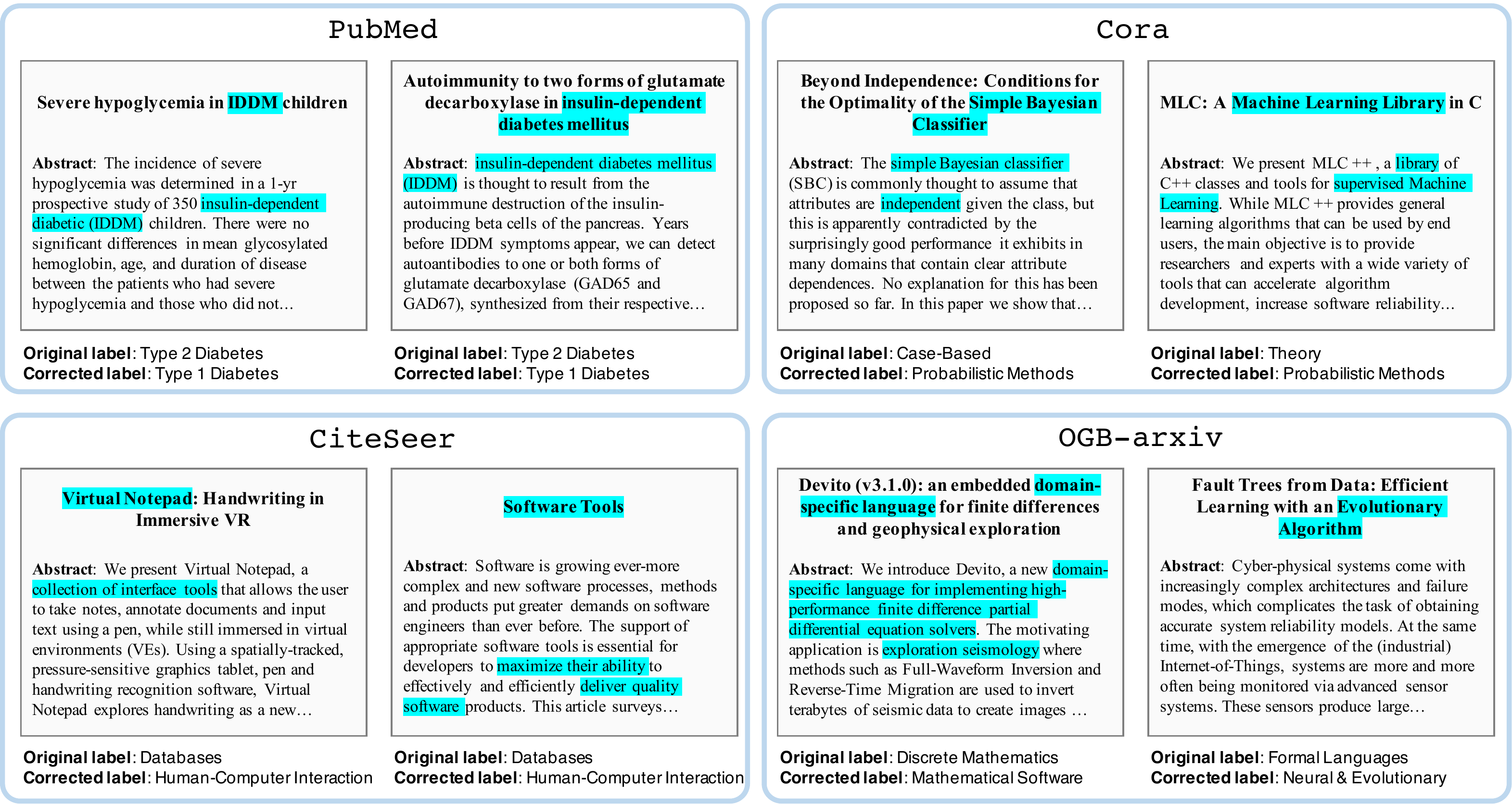}
    \caption{Examples of mislabelled samples detected by our approach in \texttt{PubMed}, \texttt{Cora}, \texttt{CiteSeer}, and \texttt{OGB-arxiv} respectively. Below each case, we give the original label given in the dataset, along with the corrected label suggested by our algorithm. We verify that all these cases are indeed mislabels, and highlight in blue the text that provides evidence for this. For \texttt{PubMed}, note that Insulin-Dependent Diabetes Mellitus (or IDDM) is synonymous with Type 1 Diabetes.
    }
    \label{fig:mislabels}
\end{figure*}

Recent work~\cite{northcutt2021pervasive} has shown that label errors are prevalent across machine learning benchmarks: they find a label error rate of 3.4\% on average, across 10 of the most popular natural language processing, computer vision, and audio datasets. However, label errors remain unexplored in the \emph{graph} setting. Graphs are widely used for representing entities and the relationships between them, with numerous applications such as molecules, financial networks, and social networks~\cite{wu2020comprehensive}. 
In this paper, we focus on the node classification setting, one of the most common graph learning tasks, 
and aim to answer the questions: \emph{are mislabels also common in this setting}, and \emph{can they be automatically detected and corrected}?



Existing solutions for mislabel detection \cite{arazo2019unsupervised,pleiss2020identifying,northcutt2021confident,zhu2022ICML} are designed for machine learning on generic data, such as images, audio, and text, where instances are largely seen as being independent of one another. However, this means that when used in a graph setting, these methods do not effectively exploit the close \emph{neighbor-dependence} between nodes and their neighbors, which is a key characteristic of graph data. In contrast, a key idea of our approach is that strong violations in the dependency patterns between a node and its neighbors act as an important signal that the node is more likely to be mislabelled. 

Hence, in this work, we propose \textsc{GraphCleaner}, a post-hoc framework for detecting and correcting mislabelled nodes on graph datasets.
To do this, \textsc{GraphCleaner} introduces the novel ideas of \emph{1) Synthetic Mislabel Dataset Generation}, where we first estimate a `mislabel transition matrix' describing patterns of how samples of different classes tend to be mislabelled, then use this transition matrix to synthesize realistically mislabelled samples in a class-dependent way. Next, these synthesized mislabelled samples are used as negative samples to train our \emph{2) Neighborhood-Aware Mislabel Detector} component, which is a binary classifier taking as input both the observed labels and the base classifier predictions, in a node's neighborhood. In this way, \textsc{GraphCleaner} exploits neighbor-dependence in graphs by learning to distinguish between the neighborhoods of correctly labelled and mislabelled nodes.

To evaluate its real-world utility and understand the implications of label errors, we conduct detailed case studies. On \texttt{PubMed}, we find that at least 6.91\% of the data is mislabelled or ambiguous, and simply removing these mislabelled data boosts evaluation performance from 86.71\% to 89.11\%. This validates the importance of data quality for performance and evaluation, and suggests that correcting mislabels has scope for significant value. Then, as shown in Figure \ref{fig:mislabels}, we apply our method to four widely used datasets to detect mislabels. To show the labels suggested by our algorithm are indeed correct, we verify in each case and provide evidence as highlighted in blue. Automatically detecting such samples can greatly speed up the efficiency of human manual checking in finding and correcting mislabels, and hence is our main goal in this work.

Our contributions can be summarized as follows:

\begin{compactitem}
\item We propose \textsc{GraphCleaner} to detect mislabels in graph data, and prove the theoretical guarantees regarding the mislabel score threshold. Unlike existing approaches, \method exploits graph's \emph{neighbor-dependence} patterns, by introducing a neighborhood-aware mislabel detector and the novel \emph{synthetic mislabel dataset generation} for better training the detector.

\item Extensive experiments on 6 graph datasets across 6 experimental settings show that \textsc{GraphCleaner} consistently outperforms state-of-the-art methods by an average margin of $0.14$ in F1 score, and $0.16$ in MCC.


\item Detailed case studies show real and previously unknown mislabels in four datasets, verifying \textsc{GraphCleaner}'s effectiveness in real-world applications. At least $6.91\%$ of \texttt{PubMed} data is found to be mislabelled or ambiguous, with significant implications for algorithm evaluation.


\item We publicly release 2 improved variants of \texttt{PubMed} dataset: \texttt{PubMedCleaned} and \texttt{PubMedMulti} for more accurate evaluation.
\end{compactitem}

\section{Related Work}
\begin{figure*}[t]
    \centering
    \includegraphics[width=\linewidth]{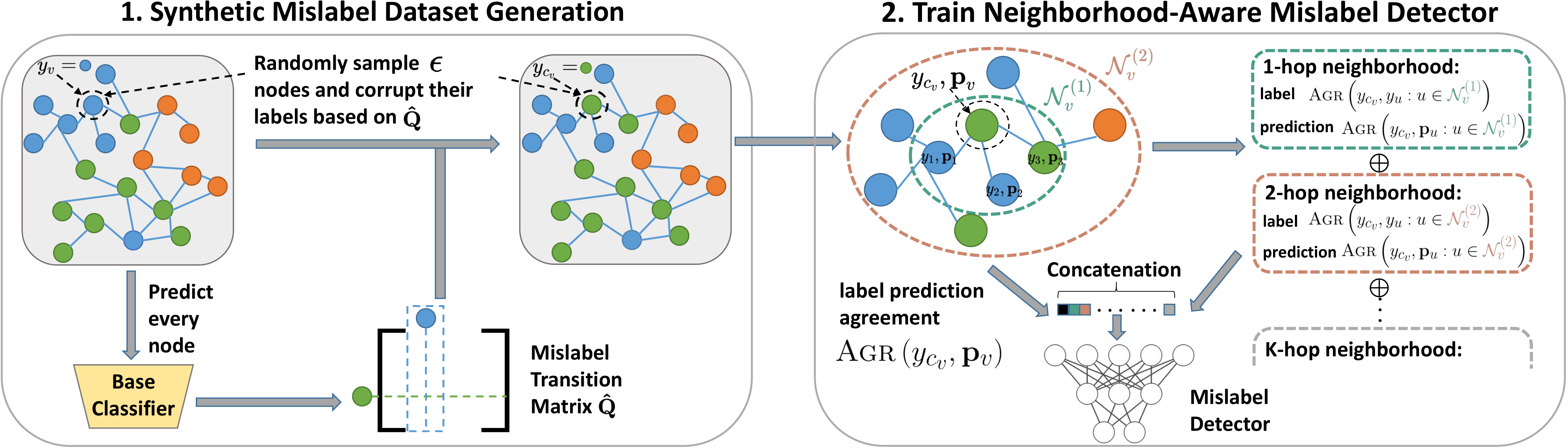}
    \caption{The framework of \textsc{GraphCleaner}. Different colors indicate different class labels. To train a mislabel detector, we first estimate the mislabel transition matrix $\hat{\mathbf{Q}}$ from the predictions of a 'base classifier', sample $\epsilon$-ratio of nodes and flip their classes $j$ to another class $i$ based on the probability of $\mathbf{\hat{Q}}_{\tilde{y}=i|y^*=j}$. Then for every node $v$, we measure the agreement between $v$ and its neighborhood within $K$ hops. The detailed design for 1-hop neighborhood $\mathcal{N}_v^{(1)}$ and 2-hop neighborhood $\mathcal{N}_v^{(2)}$ is illustrated.
    }
    \label{fig:framework}
\end{figure*}

We first review related explorations in graph neural networks, learning with noisy labels and confidence calibration. Then we summarize methods for mislabel detection. 

\paragraph{Graph Neural Networks}

Graph neural networks (GNN) have been a powerful tool for graph learning tasks. 
One line of research similar to ours utilizes label information, such as label propagation~\cite{wang2020unifying,10.1145/3394486.3403101} and integrating labels with features~\cite{shi2020masked}.
Another line is post-processing methods. Among them,  \citet{wang2021confident} tackles the graph-based calibration problem and \citet{huang2020combining} proposes the C\&S procedure. 
Due to the goal of mislabel detection, our framework differs, e.g., in its use of mislabel generation and neighbor-dependence assumption.

\paragraph{Confidence Calibration}

In machine learning, reliable uncertainty estimates are crucial for safe decision-making. However, \citet{pmlr-v70-guo17a} shows that modern neural networks suffer from \emph{over-confidence} issue, which \emph{Confidence Calibration} aims to address.
A well-calibrated model should align the output confidence score with the ground truth accuracy~\cite{Simple2016Laksh}. Popular methods include temperature scaling~\cite{pmlr-v70-guo17a}, Dirichlet calibration~\cite{kull2019beyond}, and Gaussian processes~\cite{pmlr-v108-wenger20a}, etc. 
Despite the shared interest in making decision-making safe, calibration algorithms adopt the model-centric view, while we focus on improving data quality, and we do not expect a straightforward mapping between confidence score and ground truth accuracy.

\paragraph{Learning with Noisy Labels}

Training reliable models in the presence of label noise~\cite{natarajan2013learning,algan2021image} is a closely related topic to mislabel detection, usually via modified training procedures~\cite{hoang2019learning,dai2021nrgnn}.
Some also explore neighborhood similarity~\cite{zhu2021clusterability, zhu2021second}, or utilize synthetic mislabels~\cite{xia2019anchor, jiang2021information}. 
Our \textsc{GraphCleaner} differs from them in its focus on non-i.i.d graph data, and how it obtains and use transition probabilities. 
Please refer to Appendix \ref{app:cmp_learn_with_noise} for specific differences.

\paragraph{Mislabel Detection}

Existing solutions for mislabel detection are mostly designed for generic data and typically come from three perspectives: training dynamics-based~\cite{arazo2019unsupervised,pleiss2020identifying}, joint distribution estimation of noisy and true labels~\cite{northcutt2021confident}, and neighborhood similarity~\cite{zhu2022ICML}.
Our method draws inspiration from \citet{northcutt2021confident} when estimating the mislabel transition matrix, 
but focuses on the graph setting and utilizes the neighborhood dependence property of graphs to deliver a post-hoc, plug-and-play solution.

\section{Proposed Framework}
\label{sec:method}

\subsection{Problem Definition} 

We aim to detect mislabelled nodes in a graph $G = (\mathcal{V}, \mathbf{X}, \mathbf{A})$, with node set $\mathcal{V}$ containing $n$ nodes, node feature matrix $\mathbf{X} \in \mathbb{R}^{n \times d}$ and adjacency matrix $\mathbf{A} \in \mathbb{R}^{n\times n}$. We assume that $G$ is undirected for simplicity, but our method can generalize straightforwardly to directed graphs. The unknown true label of node $v$ is denoted by $y_v^* \in [c]$, where $[c]:=\{1,2,\cdots,c\}$ is the class label set; the observed noisy label is denoted by $\tilde{y}_v \in [c]$. The node $v$ is said to be mislabelled if $y_v^* \neq \tilde{y}_v$. The node set $\mathcal{V}$ is partitioned into training, validation, and test sets, denoted by $\mathcal{V}_{\text{train}}$, $\mathcal{V}_{\text{val}}$, and $\mathcal{V}_{\text{test}}$.

Given a graph $G$, we aim to answer the following two fundamental questions: 
\begin{compactitem}
    \item \textbf{Identify}: Which nodes are mislabelled?
    \item \textbf{Correct}: What are their unknown true labels $y^*$?
\end{compactitem}

\paragraph{Post-hoc Setting} Our work focuses on the post-hoc setting, where we are given a \emph{pretrained base classifier} $f_\theta(\mathbf{X}, \mathbf{A})$, e.g., any graph neural network used as a node classifier, which is trained on the training set $\mathcal{V}_{\text{train}}$ with the observed noisy labels $\tilde{y}$. 

\subsection{\method: Overview}

In order to determine if a sample is mislabelled, we need a good mislabel detector that is able to capture the behavior of mislabels. We specify this detector as a neural network, which leads to the question of how to construct a training dataset representative of the real world. Unlike other problems, such as out-of-distribution detection, where the outlier distribution is unknown, mislabel detection has the advantage that the outlier distribution is well-defined and can be easily simulated by flipping the labels. Motivated by this, we tackle the mislabel detection task through two steps: \emph{generating synthetic mislabel dataset} and \emph{training neighborhood-aware mislabel detector}. The framework is illustrated in Figure \ref{fig:framework}. 

To obtain mislabelled data that captures the characteristics of real mislabels, we first estimate the mislabel transition matrix $\mathbf{\hat{Q}}$ using a base classifier. 
The mislabel transition matrix is then utilized to generate the balanced corrupted dataset.
Next, the \emph{Mislabel Detector} is trained on the corrupted dataset by capturing neighborhood-aware features.

\subsection{Synthetic Mislabel Dataset Generation}

There is no ground truth about if a sample is mislabelled. To train a mislabel detector, we first need to synthesize mislabelled samples to obtain ground truth labels. Our approach generates synthetic mislabels in a \emph{class-aware} way: the label noise is class-dependent, meaning that the probability of mislabelling a node as class $i$ is dependent on the node's actual class $j$. For example, an image of a little tiger is more easily mislabelled as a cat than a dog. The overall procedure is summarized in Algorithm \ref{alg:syn-mislabel}. 

\paragraph{Learning Mislabel Transition Matrix}
To generate class-dependent label noise, we first learn a \emph{mislabel transition matrix} $\mathbf{\hat{Q}}$ to capture the probability of mislabelling one class label to another class. Specifically, $\mathbf{\hat{Q}}_{\tilde{y}=i|y^*=j}$ denotes the probability of a sample with unknown true class $j$ being mislabelled with an observed class $i$.


We estimate the mislabel transition matrix using a pretrained base classifier evaluated on validation set data, following the same assumption outlined in \citet{northcutt2021confident}. The underlying assumption is that predictions given by the base classifier with \emph{high confidence} are very likely to match the true labels. Therefore, we consider predictions with confidence scores no less than a threshold (see Appendix~\ref{sec:app} Equation~\ref{eq:threshold}) to be correct. Under this assumption, we count the number of samples with true label $j$ and observed noisy label $i$ as $\mathbf{C}_{\tilde{y}=i,y^*=j}$. This joint distribution can then be transformed into an estimate of the mislabel transition matrix $\mathbf{\hat{Q}}_{\tilde{y}=i|y^*=j}$, by an application of Bayes' theorem. We refer you to Appendix~\ref{sec:app} for more details and formal definitions of this process.

\paragraph{Sampling Mislabels}
To generate a balanced training set for training our mislabel detector, we uniformly sample half of the nodes in the validation set to be synthetic mislabels. We denote the set of sampled nodes as $\mathcal{V}_{\text{synth}} \subseteq \mathcal{V}_{\text{val}}$.

Then, nodes in $\mathcal{V}_{\text{synth}}$ will have their labels randomly flipped according to the mislabel transition matrix: specifically, the probability of changing a node's given label from $j$ to a different label $i$ is $\mathbf{\hat{Q}}_{\tilde{y}=i|y^*=j}$. Effectively, this process simulates realistic label noise, using our best estimate of the distribution of label noise in the original data. We denote the resulting \emph{corrupted label matrix} as $\mathbf{Y}_c \in \mathbb{R}^{n \times c}$, which is the one-hot label matrix after nodes in $\mathcal{V}_{\text{synth}}$ have been flipped according to the above process. For further discussion, see Appendix \ref{app:discuss_samp_mislabel}.

\begin{algorithm}[h]
   \caption{Synthetic Mislabel Dataset Generation}
   \label{alg:syn-mislabel}
   {\bfseries Input:} graph $G\!=\!(\mathcal{V}, \mathbf{X}, \mathbf{A})$ with $\mathcal{V}\!=\!(\mathcal{V}_{train}, \mathcal{V}_{val}, \mathcal{V}_{test})$, base classifier $f_\theta$, sample ratio $\epsilon$.

   {\bfseries Output:} corrupted graph $G_c = (\mathcal{V}_c, \mathbf{X}, \mathbf{A})$ with nodes in $\mathcal{V}_{\text{synth}}$ having flipped labels
   
   {\bfseries Notations:} 
   confident joint $\mathbf{C}_{\tilde{y},y^*}$, threshold $t_c$ for class $c$, noise transition matrix $\mathbf{\hat{Q}}_{\tilde{y}|y^*}$.
   
\begin{algorithmic}[1]
   \STATE Train $f_\theta$ on $\mathcal{V}_{train}$
   \STATE Calculate $t_c = \frac{1}{\left| \mathbf{X}_{\tilde{y}=c} \right|} \mathop{\sum}\limits_{x \in \mathbf{X}_{\tilde{y}=c}} \hat{p} \left( \tilde{y}=c ; x, \theta \right)$, where $\hat{p}$ is $f_\theta$'s predictions
   \FOR{$v$ {\bfseries in} $\mathcal{V}_{val}$}
   \STATE $y_v \leftarrow f_\theta(v)$ 
   \STATE $P_v \leftarrow$ $f_\theta$'s softmax prediction on $v$
   \IF{$y_v \neq \tilde{y}_v$ {\bfseries and} $P_v(y_v) \geq t_{y_v}$}
   \STATE $\mathbf{C}_{\tilde{y}=\tilde{y}_v,y^*=y_v}  \mathrel{+}= 1$
   \ENDIF
   \ENDFOR
   \STATE $\mathbf{\hat{Q}}_{\tilde{y},y^*} \leftarrow$ normalized $\mathbf{C}_{\tilde{y},y^*}$ \hfill \cmt{$\triangleright$ joint distribution}
   \STATE $\mathbf{\hat{Q}}_{\tilde{y}|y^*} \leftarrow \mathbf{\hat{Q}}_{\tilde{y},y^*} / \mathbf{\hat{Q}}_{y^*}$ \hfill \cmt{$\triangleright$ conditional probability}
   \STATE $\mathcal{V}_{\text{synth}} \leftarrow$ sampled data from $\mathcal{V}_{val}$ by $\epsilon$
   \FOR{$v$ {\bfseries in} $\mathcal{V}_{\text{synth}}$}
   \STATE Flip $v$'s label according to the transition probability distribution $\mathbf{\hat{Q}}_{\tilde{y}|y^*=\tilde{y}_{v}}$
   \ENDFOR
\end{algorithmic}
\end{algorithm}

\subsection{Neighborhood-Aware Mislabel Detector}
\label{sec:mislabel_detector}

To better detect mislabelled data, we leverage a useful and commonly-held assumption of \emph{neighbor-dependence}: the ground truth label of one node tends to agree with the labels of its neighbors.
That is, if a node strongly disagrees with its neighborhood in terms of labels, the node has a relatively high risk of being mislabelled. In addition, the base classifier's softmax predictions also carry important information about the unknown true labels. 

Motivated by these two factors, our mislabel detector intuitively focuses on the \emph{agreement between a sample's label, and the labels and base classifier predictions in its neighborhood}, to decide if the sample is mislabelled. The procedure is summarized in Algorithm \ref{alg:mislabel-detector}, and more details can be found in Appendix \ref{app:neighborhood_agreement}.

\begin{algorithm}[h]
   \caption{Neighborhood-Aware Mislabel Detector}
   \label{alg:mislabel-detector}
   {\bfseries Input:} corrupted graph $G_c = (\mathcal{V}_c, \mathbf{X}, \mathbf{A})$ with $\mathbf{D}$ as its diagonal degree matrix, softmax prediction matrix $\mathbf{P}$, original label matrix $\mathbf{Y}$, corrupted label matrix $\mathbf{Y}_c$, number of hops $K$.

   {\bfseries Output: } the trained mislabel detector
   
   {\bfseries Notations:} neighborhood agreement features $\mathbf{Z}$
   
\begin{algorithmic}[1]
   \STATE $\tilde{\mathbf{A}} \leftarrow \mathbf{D}^{-1/2} \mathbf{A} \mathbf{D}^{-1/2}$
   \STATE $\mathbf{Z} \leftarrow \agroperator (\mathbf{Y}_c, \mathbf{P}) $
   \FOR{$k = 1 \rightarrow K$}
   \STATE $\mathbf{S}_k \leftarrow \mathsf{zero}(\tilde{\mathbf{A}}^k)$
   \STATE $\mathbf{Y}^{(k)} \leftarrow \mathbf{S}_k \cdot \mathbf{Y}$
   \STATE $\mathbf{P}^{(k)} \leftarrow \mathbf{S}_k \cdot \mathbf{P}$
   \STATE $\mathbf{Z} = \mathbf{Z} \oplus \mathbf{Y}^{(k)} \oplus \mathbf{P}^{(k)}$
   \ENDFOR
   \STATE Train the mislabel detector with $\mathbf{Z}_{val}$, $\mathbf{Y}_{val_c}$
   \STATE Apply the mislabel detector on $\mathbf{Z}_{test}$
\end{algorithmic}
\end{algorithm}

\paragraph{Neighborhood Extraction}
We now describe how we extract the labels and base classifier predictions from the neighborhood of each node. Define the normalized adjacency matrix as $\tilde{\mathbf{A}}:=\mathbf{D}^{-1/2} \mathbf{A} \mathbf{D}^{-1/2}$ where $\mathbf{D}$ is the diagonal matrix of node degrees. 
\begin{definition}
The $k$-hop propagation matrix is defined as 
\begin{align}
\mathbf{S}_k := \mathsf{zero}(\tilde{\mathbf{A}}^k) \in \mathbb{R}^{n \times n},
\end{align}
where the $\mathsf{zero}(\cdot)$ operation indicates zeroing-out all the diagonal entries of the matrix. 
\end{definition}

Multiplying any signal by $\mathbf{S}_k$ propagates it over $k$-hop neighborhoods. The $\mathsf{zero}(\cdot)$ operation is essential for avoiding \emph{label leakage}, where information about the target labels influences the input to a classifier. 

Then, to extract label information in each node's $k$-hop neighborhood, we start with the label matrix $\mathbf{Y} \in \mathbb{R}^{n \times c}$, and propagate it over $k$ hops by computing $\mathbf{Y}^{(k)} := \mathbf{S}_k \cdot \mathbf{Y} \in \mathbb{R}^{n \times c}$ for any $k$. In order to extract the base classifier's predictions in such neighborhoods, we do the same propagation steps on the softmax prediction matrix $\mathbf{P} \in \mathbb{R}^{n \times c}$, obtaining $\mathbf{P}^{(k)}:=\mathbf{S}_k \cdot \mathbf{P} \in \mathbb{R}^{n \times c}$. Intuitively, $\mathbf{Y}$ and $\mathbf{P}$ are smoothed by information in $k$-hop neighborhood, yielding $\mathbf{Y}^{(k)}$ and $\mathbf{P}^{(k)}$, which have greater expression capability by encoding neighborhood information.

\paragraph{Neighborhood Agreement Features}
Recall that $\mathbf{Y}$ contains the original observed labels, while $\mathbf{Y}_c$ contains the corrupted labels where we flip the labels of nodes in $\mathcal{V}_{\text{synth}}$. 
To better exploit neighbor-dependence, we employ an agreement operator $\agroperator: \mathbb{R}^{n\times c} \times \mathbb{R}^{n\times c} \mapsto \mathbb{R}^{n\times c}$, which we will use to measure the \emph{agreement} between the (possibly corrupted) node's own label $\mathbf{Y}_c$, and three quantities from the above Neighborhood Extraction process: 1) the model predictions $\mathbf{P}$; 2) the neighborhood-propagated predictions $\mathbf{P}^{(k)}$; and 3) the neighborhood-propagated labels $\mathbf{Y}^{(k)}$, for every hop $k \in [K]$. 


There are various ways to gauge such agreement and a simple way is to take \emph{dot products} $\ominus$ as the agreement measure \agroperator. Accordingly, define the \emph{row-wise dot product}, denoted $\mathbf{U}\ominus \mathbf{V}$ for any same-sized matrices $\mathbf{U}$ and $\mathbf{V}$, as the column vector whose $i$-th entry is the dot product between the $i$-th rows of $\mathbf{U}$ and $\mathbf{V}$. 
In this way, we construct the input feature matrix $\mathbf{Z}$ to our mislabel detector by concatenating such agreement terms\footnote{Using both original and corrupted label matrices ($\mathbf{Y}$ and $\mathbf{Y}_c$) in generating the neighborhood agreement features is an important design choice to ensure that synthetic mislabels only corrupt their own features, not the features of other nodes; see Appendix \ref{app:neighborhood_agreement}.}:

\begin{scriptsize}
\begin{align*}
    \mathbf{Z} \! = \! \! \left[\! \mathbf{Y}_c \! \ominus \! \mathbf{P},   \underbrace{\mathbf{Y}_c \! \ominus \! \mathbf{P}^{(1)}, \cdots , \mathbf{Y}_c \! \ominus \! \mathbf{P}^{(K)}}_K,
    \underbrace{\mathbf{Y}_c \! \ominus \! \mathbf{Y}^{(1)} , \cdots , \mathbf{Y}_c \! \ominus \! \mathbf{Y}^{(K)}}_K
    \!\right]
\end{align*}
\end{scriptsize}


\paragraph{Mislabel Detector} We train a Multi-Layer Perceptron (MLP) model on the validation set as the mislabel detector, with input features $\mathbf{Z}$, and output as a binary variable indicating whether the node is a synthetic mislabelled node. L1 loss $L(x,y)=\frac{1}{N} \sum_{i=1}^{n} |x_i-y_i|$ is adopted for training, motivated by the finding in \citet{hu2022understanding} that L1 loss has stronger robustness and smaller calibration error than the commonly used cross entropy loss.
Mislabel detection is closely related to calibration, since both tasks expect the softmax probability associated with the predicted class label to reflect its ground truth correctness likelihood.

\paragraph{Inference} For the target graph $G = (\mathcal{V}, \mathbf{X}, \mathbf{A})$ with label matrix $\mathbf{Y}$, we construct the softmax matrix $\mathbf{P}$ using a base classifier and do $K$-hop neighborhood propagation to obtain $\mathbf{P}^{(k)}$ and $\mathbf{Y}^{(k)}$ for $k \in [K]$. Note that the corrupted label matrix $\mathbf{Y}_c$ is only needed at training time; at test time, the observed label matrix $\mathbf{Y}$ is used in its place. For every mislabelled node, its unknown true label $y^*$ is estimated using the base classifier's prediction $\hat{y}^* = \text{argmax}_i \ p(y=i \ |\ \mathbf{x})$. For the time complexity of our \method, see Appendix \ref{app:complexity}.

\paragraph{Remark} Both our mislabel transition matrix and mislabel detector are estimated on $\mathcal{V}_{\text{val}}$. There are three advantages in doing so. 1) $\mathcal{V}_{\text{val}}$ is more representative of $\mathcal{V}_{\text{test}}$, avoiding overfitting on $\mathcal{V}_{\text{train}}$; 2) We can easily adapt to distribution shift away from the training distribution; 3) We only need a small amount of data to train the mislabel detector, which requires fewer parameters and improves efficiency.

\begin{table*}[th]
\centering
\caption{Mislabel detection accuracy of our \method and other methods across 6 graph benchmarks and 6 noise settings. The percentage in the first column indicates the mislabel ratio $\epsilon$, while `sym' and `asym' refer to symmetric and asymmetric noise. The best results are emphasized in bold, and * indicates a statistically significant ($p < 0.01$) difference between the best and the second best result according to T-test. The last row gives the average improvement of \method's result over the second best. If \method is not the best, it is compared to the best method. }
\scalebox{.6}{
\begin{tabular}{l|l|lll|lll|lll|lll|lll|lll}
\toprule[1.2pt]
~ & \multirow{2}*{Method} & \multicolumn{3}{c|}{\texttt{Cora}} & \multicolumn{3}{c|}{\texttt{CiteSeer}} & \multicolumn{3}{c|}{\texttt{PubMed}} & \multicolumn{3}{c|}{\texttt{Computers}} & \multicolumn{3}{c|}{\texttt{Photo}} & \multicolumn{3}{c}{\texttt{OGB-arxiv}} \\
~ &     ~ &      F1 &     MCC &   P@T &      F1 &     MCC &   P@T &      F1 &     MCC &   P@T &      F1 &     MCC &   P@T &      F1 &     MCC &   P@T &      F1 &     MCC &   P@T \\
\midrule
\midrule
\multirow{5}*{\begin{tabular}{@{}c@{}}10\% \\ sym\end{tabular}} & baseline &  0.269 &  0.361 &  0.158 &  0.287 &  0.289 &  0.195 &  0.505 &  0.465 &  0.441 &  0.211 &  0.325 &  0.118 &  0.213 &  0.322 &  0.120 &  0.047 &  0.138 &  0.024 \\
~ & DYB      &  0.455 &  0.467 &  0.321 &  0.297 &  0.270 &  0.167 &  0.197 &  0.045 &  0.179 &  0.755 &  0.745 &  \textbf{0.843} &  0.777 &  0.770 &  0.884 &  0.500 &  0.496 &  0.705 \\
~ & AUM      &  0.235 &  0.199 &  0.620 &  0.202 &  0.121 &  0.424 &  0.183 &  0.046 &  $\textbf{0.678}^{*}$ &  0.181 &  0.026 &  0.762 &  0.179 &  0.040 &  0.839 &  0.366 &  0.365 &  0.425 \\
~ & CL       &  0.560 &  0.513 &  0.570 &  0.432 &  0.386 &  0.459 &  0.447 &  0.388 &  0.483 &  0.657 &  0.619 &  0.692 &  0.776 &  0.755 &  0.733 &  0.303 &  0.292 &  0.383 \\
~ & Ours     &  $\textbf{0.790}^{*}$ &  $\textbf{0.773}^{*}$ &  $\textbf{0.803}^{*}$ &  $\textbf{0.560}^{*}$ &  $\textbf{0.535}^{*}$ &  $\textbf{0.504}^{*}$ &  $\textbf{0.616}^{*}$ &  $\textbf{0.576}^{*}$ &  0.627 &  $\textbf{0.844}^{*}$ &  $\textbf{0.830}^{*}$ &  0.840 &  $\textbf{0.902}^{*}$ &  $\textbf{0.893}^{*}$ &  \textbf{0.897} &  $\textbf{0.760}^{*}$ &  $\textbf{0.744}^{*}$ &  $\textbf{0.809}^{*}$ \\

\midrule
\multirow{5}*{\begin{tabular}{@{}c@{}}10\% \\ asym\end{tabular}} & baseline &  0.141 &  0.219 &  0.078 &  0.153 &  0.170 &  0.094 &  0.366 &  0.324 &  0.296 &  0.006 &  0.046 &  0.003 &  0.000 &  -0.009 &  0.000 &  0.003 &  0.025 &  0.002 \\
~ & DYB      &  0.404 &  0.397 &  0.283 &  0.289 &  0.246 &  0.174 &  0.278 &  0.234 &  0.242 &  0.625 &  0.626 &  0.625 &  0.678 &  0.675 &  0.744 &  0.310 &  0.251 &  0.320 \\
~ & AUM      &  0.223 &  0.169 &  0.577 &  0.199 &  0.111 &  0.385 &  0.183 &  0.048 &  $\textbf{0.657}^{*}$ &  0.182 &  0.033 &  $\textbf{0.726}^{*}$ &  0.179 &  0.037 &  0.782 &  0.252 &  0.145 &  0.357 \\
~ & CL       &  0.556 &  0.511 &  0.542 &  0.419 &  0.365 &  $\textbf{0.431}^{*}$ &  0.456 &  0.402 &  0.431 &  0.671 &  0.635 &  0.693 &  0.760 &  0.735 &  0.737 &  0.274 &  0.232 &  0.322 \\
~ & Ours     &  $\textbf{0.669}^{*}$ &  $\textbf{0.633}^{*}$ &  $\textbf{0.647}^{*}$ &  $\textbf{0.475}^{*}$ &  $\textbf{0.424}^{*}$ &  0.398 &  $\textbf{0.565}^{*}$ &  $\textbf{0.516}^{*}$ &  0.545 &  $\textbf{0.778}^{*}$ &  $\textbf{0.757}^{*}$ &  0.692 &  $\textbf{0.832}^{*}$ &  $\textbf{0.817}^{*}$ &  $\textbf{0.797}^{*}$ &  $\textbf{0.477}^{*}$ &  $\textbf{0.415}^{*}$ &  $\textbf{0.452}^{*}$ \\

\midrule
\multirow{5}*{\begin{tabular}{@{}c@{}}5\% \\ sym\end{tabular}} &baseline &  0.185 &  0.262 &  0.108 &  0.256 &  0.278 &  0.171 &  0.443 &  0.414 &  0.453 &  0.199 &  0.317 &  0.111 &  0.147 &  0.261 &  0.081 &  0.053 &  0.152 &  0.028 \\
~ & DYB      &  0.276 &  0.340 &  0.257 &  0.190 &  0.229 &  0.062 &  0.129 &  0.100 &  0.151 &  0.585 &  0.611 &  0.699 &  0.644 &  0.670 &  0.760 &  0.275 &  0.332 &  0.612 \\
~ & AUM      &  0.132 &  0.157 &  0.532 &  0.108 &  0.093 &  0.381 &  0.095 &  0.032 &  \textbf{0.557} &  0.093 &  0.015 &  0.616 &  0.091 &  0.030 &  0.760 &  0.212 &  0.266 &  0.302 \\
~ & CL       &  0.477 &  0.456 &  0.506 &  0.324 &  0.340 &  0.273 &  0.310 &  0.303 &  0.320 &  0.572 &  0.560 &  0.628 &  0.736 &  0.723 &  0.743 &  0.165 &  0.198 &  0.264 \\
~ & Ours     &  $\textbf{0.661}^{*}$ &  $\textbf{0.671}^{*}$ &  $\textbf{0.719}^{*}$ &  $\textbf{0.406}^{*}$ &  $\textbf{0.425}^{*}$ &  \textbf{0.396} &  \textbf{0.477} &  $\textbf{0.483}^{*}$ &  0.551 &  $\textbf{0.724}^{*}$ &  $\textbf{0.728}^{*}$ &  $\textbf{0.754}^{*}$ &  $\textbf{0.841}^{*}$ &  $\textbf{0.843}^{*}$ &  $\textbf{0.818}^{*}$ &  $\textbf{0.596}^{*}$ &  $\textbf{0.615}^{*}$ &  $\textbf{0.747}^{*}$ \\

\midrule

\multirow{5}*{\begin{tabular}{@{}c@{}}5\% \\ asym\end{tabular}} & baseline &  0.157 &  0.235 &  0.089 &  0.232 &  0.257 &  0.152 &  0.252 &  0.213 &  0.251 &  0.000 &  0.000 &  0.000 &  0.000 & -0.006 &  0.000 &  0.005 &  0.032 &  0.003 \\
~ & DYB      &  0.270 &  0.329 &  0.266 &  0.174 &  0.203 &  0.062 &  0.158 &  0.173 &  0.159 &  0.509 &  0.553 &  0.545 &  0.538 &  0.578 &  0.639 &  0.216 &  0.245 &  0.304 \\
~ & AUM      &  0.125 &  0.132 &  0.500 &  0.110 &  0.100 &  0.323 &  0.095 &  0.032 &  $\textbf{0.541}^{*}$ &  0.129 &  0.066 &  0.599 &  0.091 &  0.028 &  0.742 &  0.095 &  0.004 &  0.266 \\
~ & CL       &  0.512 &  0.489 &  0.517 &  0.306 &  0.325 &  0.223 &  0.311 &  0.315 &  0.206 &  0.576 &  0.567 &  0.597 &  0.755 &  0.743 &  \textbf{0.757} &  0.154 &  0.168 &  0.223 \\
~ & Ours     &  $\textbf{0.586}^{*}$ &  $\textbf{0.590}^{*}$ &  $\textbf{0.617}^{*}$ &  $\textbf{0.364}^{*}$ &  $\textbf{0.372}^{*}$ &  \textbf{0.350} &  $\textbf{0.454}^{*}$ &  $\textbf{0.449}^{*}$ &  0.486 &  $\textbf{0.695}^{*}$ &  $\textbf{0.700}^{*}$ &  \textbf{0.608} &  \textbf{0.762} &  \textbf{0.758} &  0.733 &  $\textbf{0.457}^{*}$ &  $\textbf{0.446}^{*}$ &  $\textbf{0.435}^{*}$ \\

\midrule
\multirow{5}*{\begin{tabular}{@{}c@{}}2.5\% \\ sym\end{tabular}} & baseline &  0.264 &  0.327 &  0.167 &  0.223 &  0.215 &  0.191 &  \textbf{0.393} &  0.392 &  0.463 &  0.163 &  0.272 &  0.091 &  0.122 &  0.210 &  0.068 &  0.051 &  0.143 &  0.026 \\
~ & DYB      &  0.133 &  0.220 &  0.162 &  0.090 &  0.147 &  0.000 &  0.086 &  0.140 &  0.204 &  0.321 &  0.404 &  0.382 &  0.349 &  0.440 &  0.441 &  0.137 &  0.221 &  0.440 \\
~ & AUM      &  0.061 &  0.103 &  0.448 &  0.054 &  0.078 &  \textbf{0.250} &  0.048 &  0.021 &  \textbf{0.596} &  0.046 &  0.012 &  0.489 &  0.044 &  0.018 &  0.747 &  0.114 &  0.189 &  0.203 \\
~ & CL       &  0.402 &  0.409 &  0.476 &  0.166 &  0.216 &  0.186 &  0.221 &  0.274 &  0.300 &  0.424 &  0.446 &  0.511 &  0.698 &  0.699 &  0.738 &  0.085 &  0.138 &  0.169 \\
~ & Ours     &  $\textbf{0.445}^{*}$ &  $\textbf{0.492}^{*}$ &  $\textbf{0.586}^{*}$ &  \textbf{0.237} &  $\textbf{0.303}^{*}$ &  0.245 &  0.385 &  $\textbf{0.452}^{*}$ &  0.562 &  $\textbf{0.555}^{*}$ &  $\textbf{0.594}^{*}$ &  $\textbf{0.658}^{*}$ &  \textbf{0.709} &  \textbf{0.729} &  \textbf{0.753} &  $\textbf{0.417}^{*}$ &  $\textbf{0.484}^{*}$ &  $\textbf{0.626}^{*}$ \\

\midrule
\multirow{5}*{\begin{tabular}{@{}c@{}}2.5\% \\ asym\end{tabular}} & baseline &  0.136 &  0.183 &  0.081 &  0.112 &  0.108 &  0.086 &  0.237 &  0.226 &  0.225 &  0.000 &  0.000 &  0.000 &  0.000 & -0.004 &  0.000 &  0.005 &  0.032 &  0.003 \\
~ & DYB      &  0.139 &  0.234 &  0.167 &  0.070 &  0.099 &  0.000 &  0.083 &  0.125 &  0.125 &  0.328 &  0.421 &  0.377 &  0.320 &  0.409 &  0.344 &  0.127 &  0.197 &  0.250 \\
~ & AUM      &  0.063 &  0.104 &  0.424 &  0.053 &  0.073 &  \textbf{0.218} &  0.063 &  0.049 &  $\textbf{0.583}^{*}$ &  0.069 &  0.048 &  0.488 &  0.057 &  0.045 &  0.709 &  0.055 &  0.021 &  0.184 \\
~ & CL       &  0.411 &  0.418 &  0.438 &  0.162 &  0.207 &  0.145 &  0.225 &  0.289 &  0.208 &  0.428 &  0.456 &  0.471 &  $\textbf{0.698}^{*}$ &  $\textbf{0.700}^{*}$ &  $\textbf{0.712}^{*}$ &  0.083 &  0.121 &  0.146 \\
~ & Ours     &  $\textbf{0.502}^{*}$ &  $\textbf{0.566}^{*}$ &  $\textbf{0.591}^{*}$ &  $\textbf{0.214}^{*}$ &  $\textbf{0.262}^{*}$ &  0.173 &  $\textbf{0.334}^{*}$ &  $\textbf{0.384}^{*}$ &  0.417 &  $\textbf{0.572}^{*}$ &  $\textbf{0.621}^{*}$ &  $\textbf{0.558}^{*}$ &  0.669 &  0.683 &  0.600 &  $\textbf{0.365}^{*}$ &  $\textbf{0.412}^{*}$ &  $\textbf{0.406}^{*}$ \\

\midrule
\multicolumn{2}{c|}{improvement} &  +0.122 &   +0.155 & +0.134 &    +0.065 &    +0.080 &    +0.001 &    +0.081 &    0.097 & -0.071 &    +0.121 &    +0.128 &    +0.041 &    +0.049 &    +0.059 &   -0.007 & +0.251 &    +0.229 &    +0.135 \\
\bottomrule[1.2pt]
\end{tabular}
}
\label{tab:basic_experiment}
\end{table*}

\subsection{Theoretical Guarantees}
In order to assist users in selecting appropriate thresholds for converting mislabel scores into binary mislabel predictions, we propose an algorithm that is accompanied by theoretical guarantees on false positive and false negative rates, derived from the conformal prediction framework~\cite{vovk2005algorithmic,balasubramanian2014conformal}. That is, if we choose the threshold based on the following propositions, the probability of a false positive (i.e. mistakenly classifying a correctly labelled sample as mislabelled) and a false negative (i.e. mistakenly classifying a mislabelled sample as correctly labelled) can be bounded by the user-defined confidence level $\alpha$.


Specifically, given a dataset $\{(\mathbf{x}^{(i)},y^{(i)}) \}_{i=1}^{N}$ with mislabelling rate $p$ (i.e. the total number of mislabeled samples is $Np$), we compute every sample's corresponding mislabel scores $\{s^{(i)} \}_{i=1}^{N}$ and sort them in non-decreasing order and denote them as $(s_{(1)}, \dots, s_{(N)})$. We can have following theoretical guarantees:  
\begin{proposition}[False Positive Guarantee]
For any given confidence level $\alpha \in (\frac{1}{N+1}, 1)$, define the threshold as $\lambda_\alpha := s_{(B_\alpha)}$, \text{ where } $B_\alpha = \left\lceil(N(1-p)+1)(1- \alpha) + Np \right\rceil$, with probability at least $1-\alpha$ over the random choice of a correctly labelled sample $(\Tilde{\mathbf{x}}, \Tilde{y})$, we have:
\begin{align*}
\tilde{s} \le \lambda_\alpha,
\end{align*}
\end{proposition}

\begin{proposition}[False Negative Guarantee]
Define the modified score function $s'\! :=\! (1-s)\cdot \mathds{1}_{\{ s > 0.5 \}}$. For any given confidence level $\alpha \in (\frac{1}{N+1}, 1)$, define the threshold as $\lambda_\alpha\! :=\! s'_{(B_\alpha)}$, \text{ where } $B_\alpha\! =
\!\left\lceil(Np+1)(1- \alpha) + N(1-p) \right\rceil$, with probability at least $1-\alpha$ over the random choice of a mislabelled sample $(\Tilde{\mathbf{x}}, \Tilde{y})$ and with $\Tilde{s}'$ as its modified score, we have:
\begin{align*}
\Tilde{s}' = (1-\tilde{s})\cdot \mathds{1}_{\{ \tilde{s} > 0.5 \}} \le \lambda_\alpha,
\end{align*}
\end{proposition}

The detailed proof is in Appendix \ref{app:mislabel-proof}. 

\paragraph{Remark} 
These two propositions show that by sorting the mislabel scores in non-decreasing order and selecting a threshold based on the user-specified confidence level $\alpha$, we can ensure that the probability of falsely classifying a correctly labelled sample as mislabelled or a mislabelled sample as correctly labelled is bounded by $\alpha$. This allows users to have confidence in the algorithm's ability to effectively detect mislabelled samples for their specific use case.





\section{Experiments}
\label{sec:exp}
In this section, we conduct experiments to answer the following research questions:
\begin{compactitem}
\item \textbf{RQ1 (Effectiveness):} Does our method outperform other state-of-the-art
methods in detecting mislabels across multiple settings, datasets, and base classifiers?
\item \textbf{RQ2 (Ablation and Robustness):} How do different components and hyperparameters of our method contribute to the performance?
\item \textbf{RQ3 (Case Studies):} Does our method detect real, previously unknown mislabels in existing popular graph learning datasets such as \texttt{PubMed}, \texttt{Cora}, \texttt{CiteSeer}, and \texttt{OGB-arxiv}? 
\end{compactitem}

Due to space limitation, we refer experiments about hyperparameters and additional discussion on case studies to Appendix \ref{app:hyperparameters} and \ref{app:case_study}.

\subsection{RQ1. Effectiveness}
\label{sec:accuracy}

\paragraph{Experimental Setup}
There is no ground truth dataset for the mislabel detection task. In order to derive datasets with ground truth labels indicating whether a sample is mislabelled, we follow the practice of related works (e.g., DYB, AUM, CL) and randomly introduce mislabels at $\epsilon$-fraction of the nodes.
Similarly to INCV~\cite{chen2019understanding}, we introduce two mislabel types in our experiment, \emph{symmetric} and \emph{asymmetric}. In the symmetric setting,  the probability of shift from one class to another is equal. In the asymmetric setting, we change class $i$ to class $(i+1) \mod c$. 

We use three mislabel rates, $\epsilon=0.1, 0.05, 0.025$, for realistic concern. Mislabel rates are typically very low in real-world datasets. In \citet{northcutt2021pervasive}, the averaged test set mislabel rate is about 3.4\%, and the highest mislabel rate is 10.12\% on ImageNet. Though many other methods use larger mislabel rates, we choose 2.5\%, 5\% and 10\% to mimic real-world settings. 
We refer to Appendix \ref{app:exp} for more details about experimental setup.


\paragraph{Methods in Comparison} We compare \method with four other methods: AUM~\cite{pleiss2020identifying}, DYB~\cite{arazo2019unsupervised}, CL~\cite{northcutt2021confident} and a simple baseline that treats samples whose argmax predictions differ from given labels as mislabels. To the best of our knowledge, AUM, DYB and CL are the current state-of-the-art approaches for mislabel detection on non-graph data. We apply them to graphs by using a graph-based base classifier. Moreover, we only compare with the label noise modeling part of DYB as our goal is to detect mislabels.

\paragraph{Datasets and Evaluation Metrics} We use 6 datasets, namely, \texttt{Cora}, \texttt{CiteSeer} and \texttt{PubMed}~\cite{yang2016revisiting}, \texttt{Computers} and \texttt{Photo}~\cite{shchur2018pitfalls}, \texttt{OGB-arxiv}~\cite{hu2020ogb} and 3 commonly used metrics: F1,
Matthews Correlation Coefficient (MCC)
and P@T\footnote{\url{https://en.wikipedia.org/wiki/Evaluation_measures_(information_retrieval)\#Precision_at_k}}
, where `T' represents the number of artificially mislabelled nodes, which varies in different experimental settings.

\begin{table}[ht]
\centering
\caption{Performance on two different base GNN classifiers: GIN and GraphUNet (referred as `GUN'). 
We refer to Table~\ref{tab:basic_experiment} for the full name of abbreviations. 
Experiments are done under the 10\% symmetric setting.}
\scalebox{.59}{
\begin{tabular}{l|r|rrr|rrr|rrr}
\toprule[1.2pt]
~ & \multirow{2}*{Method} & \multicolumn{3}{c|}{\texttt{Cora}} & \multicolumn{3}{c|}{\texttt{Computers}} & \multicolumn{3}{c}{\texttt{OGB-arxiv}} \\
~ & ~ &       F1 &     MCC &   P@T &      F1 &     MCC &   P@T &      F1 &     MCC &   P@T \\
\midrule
\midrule
\multirow{5}*{GIN} & baseline &  0.161 &  0.083 &  0.144 &  0.176 &  0.194 &  0.108 &  0.045 &  0.140 &  0.023 \\
~ & DYB      &  0.199 &  0.104 &  0.142 &  0.441 &  0.438 &  0.603 &  0.374 &  0.364 &  0.548 \\
~ & AUM      &  0.209 &  0.147 &  0.449 &  0.408 &  0.409 &  0.644 &  0.267 &  0.237 &  0.441 \\
~ & CL       &  0.339 &  0.258 &  0.326 &  0.411 &  0.344 &  0.531 &  0.234 &  0.134 &  0.353 \\
~ & Ours     &  \textbf{0.784} &  \textbf{0.762} &  \textbf{0.756} &  \textbf{0.855} &  \textbf{0.840} &  \textbf{0.851} &  \textbf{0.687} &  \textbf{0.669} &  \textbf{0.693} \\

\midrule
\multirow{5}*{GUN} & baseline &  0.300 &  0.355 &  0.186 &  0.169 &  0.277 &  0.093 &  0.051 &  0.141 &  0.026 \\
~ & DYB      &  0.269 &  0.254 &  0.227 &  0.516 &  0.513 &  \textbf{0.831} &  0.232 &  0.185 &  0.404 \\
~ & AUM      &  0.308 &  0.307 &  0.732 &  0.555 &  0.555 &  0.684 &  0.331 &  0.318 &  0.365 \\
~ & CL       &  0.520 &  0.471 &  0.546 &  0.555 &  0.514 &  0.607 &  0.258 &  0.230 &  0.358 \\
~ & Ours     &  \textbf{0.804} &  \textbf{0.788} &  \textbf{0.845} &  \textbf{0.846} &  \textbf{0.832} &  \textbf{0.831} &  \textbf{0.691} &  \textbf{0.668} &  \textbf{0.645} \\
\bottomrule[1.2pt]
\end{tabular}
}
\label{tab:vary_GNN}
\end{table}

\paragraph{Findings} Our \method outperforms other methods under almost all metrics across 6 datasets and 6 noise settings, especially in F1 and MCC. We obtain an average margin of 0.14 in F1 and 0.16 in MCC compared to the second best method as shown in Table~\ref{tab:basic_experiment}. Table~\ref{tab:vary_GNN} shows that our \method generalizes well to different base classifiers and consistently outperforms other methods.

    
    

\subsection{RQ2. Ablation and Robustness}
\label{sec:ablation}

\paragraph{Ablation} To show the effectiveness of each component, we compare \method with three ablated variants:
\begin{compactitem}
    \item \emph{L only}: only the agreement between a sample's label and the labels of its neighbors is used; 
    
    \item \emph{P only}: only the agreement between a sample's label and the softmax predictions of its neighbors is used;

    \item \emph{No CL}: instead of following the mislabel transition matrix, we randomly mislabel sampled nodes to some other classes to generate synthetic mislabels.
\end{compactitem}

\begin{table}[htbp]
\centering
\caption{Ablation study. 
We refer to Section \ref{sec:ablation} for the description of different variants: `No CL', `L only' and `P only'. 
Experiments are conducted using GCN with 10\% noise.
}
\scalebox{0.6}{
\begin{tabular}{l|r|rrr|rrr|rrr}
\toprule[1.2pt]
~ & \multirow{2}*{Method} & \multicolumn{3}{c|}{\texttt{Cora}} & \multicolumn{3}{c|}{\texttt{Computers}} & \multicolumn{3}{c}{\texttt{OGB-arxiv}} \\
~ & ~ &       F1 &     MCC &   P@T &      F1 &     MCC &   P@T &      F1 &     MCC &   P@T \\
\midrule
\midrule
\multirow{4}*{sym} &  \emph{L only} &  0.798 &  0.780 &  0.801 &  0.840 &  0.826 &  0.845 &  0.750 &  0.733 &  0.808 \\
~ & \emph{P only} &  0.676 &  0.654 &  0.680 &  0.844 &  0.829 &  0.838 &  0.736 &  0.718 &  0.805 \\
~ & \emph{No CL}   &  \textbf{0.802} &  \textbf{0.782} &  \textbf{0.808} &  \textbf{0.865} &  \textbf{0.851} &  \textbf{0.859} &  \textbf{0.819} &  \textbf{0.799} &  \textbf{0.813} \\
~ & Ours          &  0.790 &  0.773 &  0.803 &  0.844 &  0.830 &  0.840 &  0.760 &  0.744 &  0.809 \\

\midrule
\multirow{4}*{asym} &  \emph{L only} &  0.658 &  0.624 &  0.642 &  \textbf{0.781} &  \textbf{0.761} &  0.685 &  0.460 &  0.397 &  0.438 \\
~ & \emph{P only} &  0.544 &  0.508 &  0.554 &  0.759 &  0.740 &  0.661 &  \textbf{0.481} &  \textbf{0.420} &  \textbf{0.455} \\
~ & \emph{No CL}   &  0.635 &  0.596 &  0.621 &  0.474 &  0.449 &  0.382 &  0.295 &  0.288 &  0.203 \\
~ & Ours          &  \textbf{0.669} &  \textbf{0.633} &  \textbf{0.647} &  0.778 &  0.757 &  \textbf{0.692} &  0.477 &  0.415 &  0.452 \\
\bottomrule[1.2pt]
\end{tabular}}
\label{tab:ablation}
\end{table}

Table~\ref{tab:ablation} shows that our \method consistently outperform \emph{No CL} version with a large margin in asymmetric settings, and is slightly worse in symmetric settings. This is unsurprising and reasonable considering that asymmetric mislabelling noise follows a class-wise pattern that can be dealt with more effectively using the mislabel transition matrix, while symmetric setting follows a class-irrelevant noise pattern. Overall, this suggests that our default framework with mislabel transition matrix is more robust to different environments, e.g. symmetric and asymmetric. 
Moreover, when compared to \emph{L only} and \emph{P only}, our default \method always outperforms the inferior method, and in most cases outperforms both, indicating that our method is adaptive and can generalize better.




\subsection{RQ3. Case Studies}

\paragraph{Overview} In this subsection, we aim to show that our method is able to detect and correct real, previously unknown mislabelled samples in a variety of popular real-world graph datasets: \texttt{PubMed}, \texttt{CiteSeer}, \texttt{Cora}, and \texttt{OGB-arxiv}. Then, to understand the implications of these findings, we conduct a more detailed analysis on the \texttt{PubMed} dataset, where some auxiliary information allows us to loosely estimate the number and impact of mislabels.

\paragraph{Experimental Procedure} We follow the same experimental settings in Section \ref{sec:accuracy}, except applied to the original datasets without synthetic mislabels.
For each dataset, we extract the top $30$ samples which our algorithm assigns the highest mislabel scores, i.e., that it regards as most likely to be mislabelled. We then manually inspect each of these top $30$ samples from each dataset by accessing the original text of each paper, to determine whether our algorithm is correct. In particular, we manually categorize each sample into 5 possible categories: `clear mislabel', `likely mislabel', `ambiguous', `likely non-mislabel', and `clear non-mislabel'. 

\paragraph{Findings} The results are shown in Figure \ref{fig:casestudy_summary}. We have a number of key findings:
\begin{compactitem}
\item A substantial fraction of these samples are mislabelled: averaged over the 4 datasets, $39\%$ of the samples are likely or clear mislabels. We do not expect this fraction to be close to $100\%$, since the fraction of mislabels overall is expected to be low~\cite{northcutt2021pervasive}; moreover, our base classifiers are not close to $100\%$ accurate. Still, the results indicate that our method can greatly improve the efficiency of manual corrections, by identifying likely mislabels for humans to check. 

\item There is a fairly large variation between datasets, with the largest fraction of mislabels in \texttt{PubMed}. This may be due to several factors: differences in the fraction of mislabels in each dataset, different accuracy of the base classifiers, and different levels of label ambiguity. 

\item In practice, there can also be value in identifying the `ambiguous' or even `likely non-mislabel' samples, as such samples carry a significant amount of \emph{label noise}. Many such samples span multiple categories
, where algorithms for handling label noise could be employed.
\end{compactitem}

\begin{figure}[h]
    \centering
    \includegraphics[width=\linewidth]{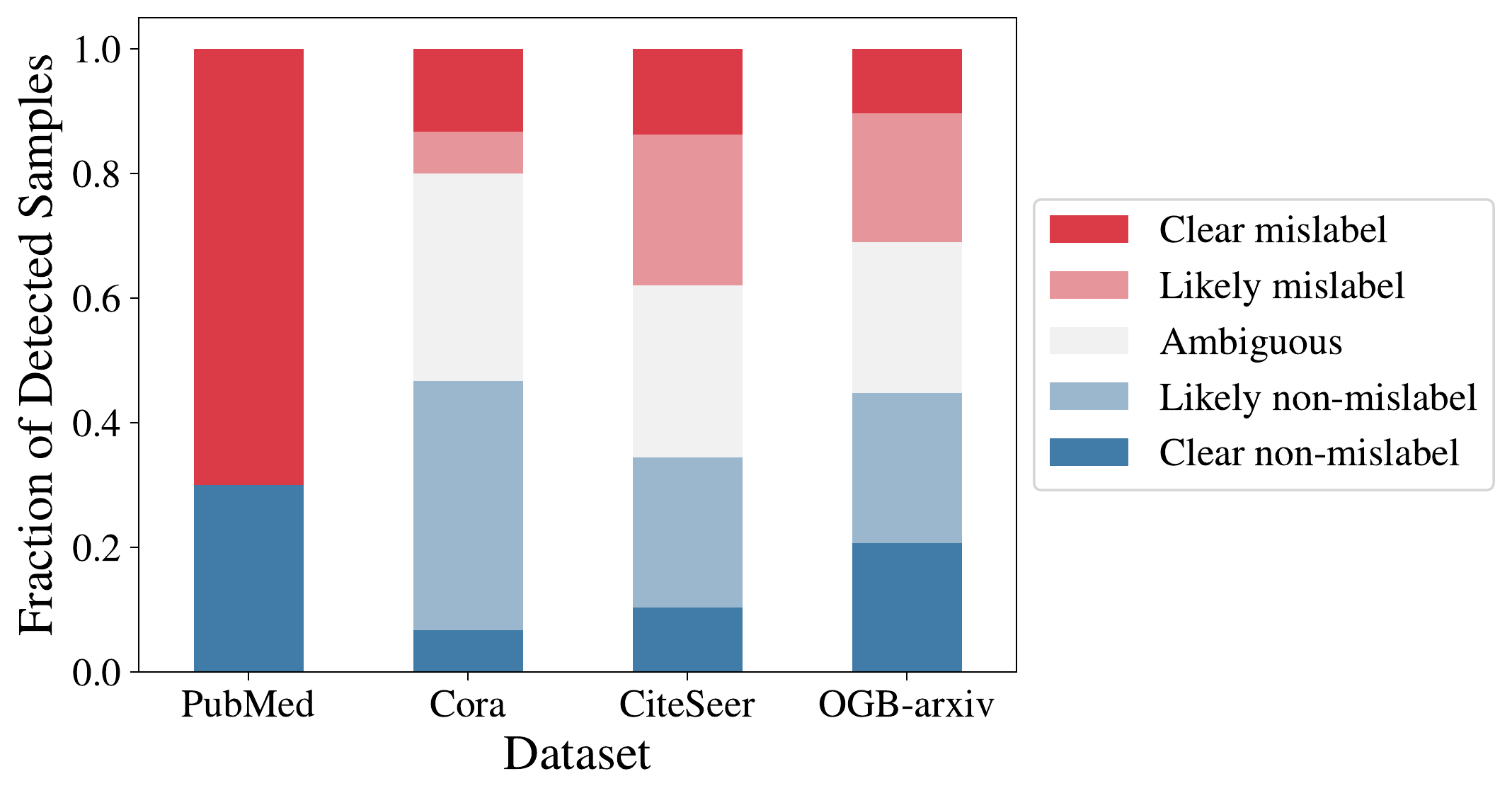}
    \caption{
    Fraction of each mislabel type among the top 30 detected samples in each dataset: 
    namely, \texttt{PubMed}, \texttt{Cora}, \texttt{CiteSeer}, and \texttt{OGB-arxiv}. }
    \label{fig:casestudy_summary}
\end{figure}


\subsubsection{Mislabel Rate Estimation on \texttt{PubMed}}
Can we loosely estimate the total amount of label noise in \texttt{PubMed}? What are the implications of this label noise? How do they affect the evaluation performance? 
To answer these questions, we focus on \texttt{PubMed} due to the presence of some helpful auxiliary information.

By querying PubMed API, we find that 71 papers of \texttt{PubMed} are assigned 0 labels (i.e., MeSH terms), 1256 papers are assigned 2 labels, and 39 papers are assigned 3 labels. However, the \texttt{PubMed} dataset was created by simply selecting the first label in alphabetical order rather than the `most correct' class. Consequently, this injects a form of label noise or ambiguity to at least $6.91\%$ of the entire dataset. More details about the estimation can be found in Appendix \ref{app:case_study}. Note that this is only a lower bound for the true amount of label noise, as it excludes mislabels from annotations and other sources, such as those in Figure \ref{fig:mislabels}.

\paragraph{Implication of Label Noise} What are the implications of these $6.91\%$ label noise samples? To study the effect, we compare the accuracy of a simple GCN+mixup baseline before and after removing the above-mentioned noisy-labelled samples in the test set. The accuracy improves from $86.71\%$ to $89.11\%$, showing that label noise has significant effects on performance evaluation. This validates the importance of data quality (e.g., label noise) for performance and evaluation, and suggests that correcting mislabels has scope for significant value. 

In addition, recognizing the importance of correcting this label noise, we publicly release two new variants of \texttt{PubMed} dataset: 1) \texttt{PubMedCleaned}, which removes\footnote{We would like to clarify that we do not actually remove detected mislabels from the graph by deleting their nodes, or recommend doing so. Instead, we recommend keeping the nodes themselves, but only removing their labels (or equivalently, removing their node index from the set of node indices constituting the training or test sets, without removing them from the graph). Most graph neural networks can be trained in a semi-supervised manner, so we can easily avoid computing the loss on these mislabelled nodes, while preserving the structure and feature information of these nodes, e.g. to avoid bridge nodes from being removed.} these noisy-labelled samples, and also corrects all detected mislabels; 2) \texttt{PubMedMulti}, which keeps multi-labelled samples but explicitly assigns them multiple labels, for users to develop algorithms which can handle the multi-labelling scenario.

\section{Conclusions}
\label{sec:conclusion}
Data quality is crucial to the success of AI systems. Our case studies, however, demonstrate that label noise is prevalent in popular graph datasets, and that performance on \textsc{PubMed} changes from 86.71\% to 89.11\% by simply removing some label noise, highlighting the importance of detecting and correcting mislabels. To address the issue, we propose \method that utilizes the \emph{neighborhood-dependence} pattern of graphs to detect mislabelled nodes in a graph.
Extensive experiments on 6 datasets across 6 noise settings verify the effectiveness and robustness of our method. The case studies further validate its practicality in real-world applications.

The current study utilizes the neighborhood-dependence property to detect mislabels and improve data quality. Besides that, it would be interesting to explore what other characteristics are associated with label and data quality. 
Going beyond the current task, we believe this framework is also promising for providing calibrated confidence and misclassification detection on graph data.

\section*{Acknowledgements}

This work was supported by NUS-NCS Joint Laboratory (A-0008542-00-00).

\nocite{langley00}

\bibliography{main}

\begin{thebibliography}{32}
\providecommand{\natexlab}[1]{#1}
\providecommand{\url}[1]{\texttt{#1}}
\expandafter\ifx\csname urlstyle\endcsname\relax
  \providecommand{\doi}[1]{doi: #1}\else
  \providecommand{\doi}{doi: \begingroup \urlstyle{rm}\Url}\fi

\bibitem[Algan \& Ulusoy(2021)Algan and Ulusoy]{algan2021image}
Algan, G. and Ulusoy, I.
\newblock Image classification with deep learning in the presence of noisy
  labels: A survey.
\newblock \emph{Knowledge-Based Systems}, 215:\penalty0 106771, 2021.

\bibitem[Arazo et~al.(2019)Arazo, Ortego, Albert, O’Connor, and
  McGuinness]{arazo2019unsupervised}
Arazo, E., Ortego, D., Albert, P., O’Connor, N., and McGuinness, K.
\newblock Unsupervised label noise modeling and loss correction.
\newblock In \emph{International conference on machine learning}, pp.\
  312--321. PMLR, 2019.

\bibitem[Balasubramanian et~al.(2014)Balasubramanian, Ho, and
  Vovk]{balasubramanian2014conformal}
Balasubramanian, V., Ho, S.-S., and Vovk, V.
\newblock \emph{Conformal prediction for reliable machine learning: theory,
  adaptations and applications}.
\newblock Newnes, 2014.

\bibitem[Chen et~al.(2019)Chen, Liao, Chen, and Zhang]{chen2019understanding}
Chen, P., Liao, B.~B., Chen, G., and Zhang, S.
\newblock Understanding and utilizing deep neural networks trained with noisy
  labels.
\newblock In \emph{International Conference on Machine Learning}, pp.\
  1062--1070. PMLR, 2019.

\bibitem[Dai et~al.(2021)Dai, Aggarwal, and Wang]{dai2021nrgnn}
Dai, E., Aggarwal, C., and Wang, S.
\newblock Nrgnn: Learning a label noise resistant graph neural network on
  sparsely and noisily labeled graphs.
\newblock In \emph{Proceedings of the 27th ACM SIGKDD Conference on Knowledge
  Discovery \& Data Mining}, pp.\  227--236, 2021.

\bibitem[Guo et~al.(2017)Guo, Pleiss, Sun, and Weinberger]{pmlr-v70-guo17a}
Guo, C., Pleiss, G., Sun, Y., and Weinberger, K.~Q.
\newblock On calibration of modern neural networks.
\newblock In Precup, D. and Teh, Y.~W. (eds.), \emph{Proceedings of the 34th
  International Conference on Machine Learning}, volume~70 of \emph{Proceedings
  of Machine Learning Research}, pp.\  1321--1330. PMLR, 06--11 Aug 2017.
\newblock URL \url{https://proceedings.mlr.press/v70/guo17a.html}.

\bibitem[Hoang et~al.(2019)Hoang, Choong, and Murata]{hoang2019learning}
Hoang, N., Choong, J.~J., and Murata, T.
\newblock Learning graph neural networks with noisy labels.
\newblock 2019.

\bibitem[Hu et~al.(2022)Hu, Wang, Wang, and Li]{hu2022understanding}
Hu, T., Wang, J., Wang, W., and Li, Z.
\newblock Understanding square loss in training overparametrized neural network
  classifiers, 2022.
\newblock URL \url{https://openreview.net/forum?id=N3KYKkSvciP}.

\bibitem[Hu et~al.(2020)Hu, Fey, Zitnik, Dong, Ren, Liu, Catasta, and
  Leskovec]{hu2020ogb}
Hu, W., Fey, M., Zitnik, M., Dong, Y., Ren, H., Liu, B., Catasta, M., and
  Leskovec, J.
\newblock Open graph benchmark: Datasets for machine learning on graphs.
\newblock \emph{arXiv preprint arXiv:2005.00687}, 2020.

\bibitem[Huang et~al.(2020)Huang, He, Singh, Lim, and
  Benson]{huang2020combining}
Huang, Q., He, H., Singh, A., Lim, S.-N., and Benson, A.
\newblock Combining label propagation and simple models out-performs graph
  neural networks.
\newblock In \emph{International Conference on Learning Representations}, 2020.

\bibitem[Jia \& Benson(2020)Jia and Benson]{10.1145/3394486.3403101}
Jia, J. and Benson, A.~R.
\newblock Residual correlation in graph neural network regression.
\newblock In \emph{Proceedings of the 26th ACM SIGKDD International Conference
  on Knowledge Discovery \& Data Mining}, KDD '20, pp.\  588–598, New York,
  NY, USA, 2020. Association for Computing Machinery.
\newblock ISBN 9781450379984.
\newblock \doi{10.1145/3394486.3403101}.
\newblock URL \url{https://doi.org/10.1145/3394486.3403101}.

\bibitem[Jiang et~al.(2021)Jiang, Zhou, Liu, Li, Chen, Choi, and
  Hu]{jiang2021information}
Jiang, Z., Zhou, K., Liu, Z., Li, L., Chen, R., Choi, S.-H., and Hu, X.
\newblock An information fusion approach to learning with instance-dependent
  label noise.
\newblock In \emph{International Conference on Learning Representations}, 2021.

\bibitem[Kull et~al.(2019)Kull, Perello~Nieto, K{\"a}ngsepp, Silva~Filho, Song,
  and Flach]{kull2019beyond}
Kull, M., Perello~Nieto, M., K{\"a}ngsepp, M., Silva~Filho, T., Song, H., and
  Flach, P.
\newblock Beyond temperature scaling: Obtaining well-calibrated multi-class
  probabilities with dirichlet calibration.
\newblock \emph{Advances in neural information processing systems}, 32, 2019.

\bibitem[Lakshminarayanan et~al.(2016)Lakshminarayanan, Pritzel, and
  Blundell]{Simple2016Laksh}
Lakshminarayanan, B., Pritzel, A., and Blundell, C.
\newblock Simple and scalable predictive uncertainty estimation using deep
  ensembles.
\newblock 12 2016.

\bibitem[Langley(2000)]{langley00}
Langley, P.
\newblock Crafting papers on machine learning.
\newblock In Langley, P. (ed.), \emph{Proceedings of the 17th International
  Conference on Machine Learning (ICML 2000)}, pp.\  1207--1216, Stanford, CA,
  2000. Morgan Kaufmann.

\bibitem[Natarajan et~al.(2013)Natarajan, Dhillon, Ravikumar, and
  Tewari]{natarajan2013learning}
Natarajan, N., Dhillon, I.~S., Ravikumar, P.~K., and Tewari, A.
\newblock Learning with noisy labels.
\newblock \emph{Advances in neural information processing systems}, 26, 2013.

\bibitem[Northcutt et~al.(2021{\natexlab{a}})Northcutt, Jiang, and
  Chuang]{northcutt2021confident}
Northcutt, C., Jiang, L., and Chuang, I.
\newblock Confident learning: Estimating uncertainty in dataset labels.
\newblock \emph{Journal of Artificial Intelligence Research}, 70:\penalty0
  1373--1411, 2021{\natexlab{a}}.

\bibitem[Northcutt et~al.(2021{\natexlab{b}})Northcutt, Athalye, and
  Mueller]{northcutt2021pervasive}
Northcutt, C.~G., Athalye, A., and Mueller, J.
\newblock Pervasive label errors in test sets destabilize machine learning
  benchmarks.
\newblock 2021{\natexlab{b}}.
\newblock \doi{10.48550/ARXIV.2103.14749}.
\newblock URL \url{https://arxiv.org/abs/2103.14749}.

\bibitem[Pleiss et~al.(2020)Pleiss, Zhang, Elenberg, and
  Weinberger]{pleiss2020identifying}
Pleiss, G., Zhang, T., Elenberg, E., and Weinberger, K.~Q.
\newblock Identifying mislabeled data using the area under the margin ranking.
\newblock \emph{Advances in Neural Information Processing Systems},
  33:\penalty0 17044--17056, 2020.

\bibitem[Shchur et~al.(2018)Shchur, Mumme, Bojchevski, and
  G{\"u}nnemann]{shchur2018pitfalls}
Shchur, O., Mumme, M., Bojchevski, A., and G{\"u}nnemann, S.
\newblock Pitfalls of graph neural network evaluation.
\newblock \emph{arXiv preprint arXiv:1811.05868}, 2018.

\bibitem[Shi et~al.(2020)Shi, Huang, Feng, Zhong, Wang, and Sun]{shi2020masked}
Shi, Y., Huang, Z., Feng, S., Zhong, H., Wang, W., and Sun, Y.
\newblock Masked label prediction: Unified message passing model for
  semi-supervised classification.
\newblock \emph{arXiv preprint arXiv:2009.03509}, 2020.

\bibitem[Vovk et~al.(2005)Vovk, Gammerman, and Shafer]{vovk2005algorithmic}
Vovk, V., Gammerman, A., and Shafer, G.
\newblock \emph{Algorithmic learning in a random world}.
\newblock Springer Science \& Business Media, 2005.

\bibitem[Wang \& Leskovec(2020)Wang and Leskovec]{wang2020unifying}
Wang, H. and Leskovec, J.
\newblock Unifying graph convolutional neural networks and label propagation.
\newblock \emph{arXiv preprint arXiv:2002.06755}, 2020.

\bibitem[Wang et~al.(2021)Wang, Liu, Shi, and Yang]{wang2021confident}
Wang, X., Liu, H., Shi, C., and Yang, C.
\newblock Be confident! towards trustworthy graph neural networks via
  confidence calibration.
\newblock \emph{Advances in Neural Information Processing Systems},
  34:\penalty0 23768--23779, 2021.

\bibitem[Wenger et~al.(2020)Wenger, Kjellstr\"om, and
  Triebel]{pmlr-v108-wenger20a}
Wenger, J., Kjellstr\"om, H., and Triebel, R.
\newblock Non-parametric calibration for classification.
\newblock In Chiappa, S. and Calandra, R. (eds.), \emph{Proceedings of the
  Twenty Third International Conference on Artificial Intelligence and
  Statistics}, volume 108 of \emph{Proceedings of Machine Learning Research},
  pp.\  178--190. PMLR, 26--28 Aug 2020.
\newblock URL \url{https://proceedings.mlr.press/v108/wenger20a.html}.

\bibitem[Wu et~al.(2020)Wu, Pan, Chen, Long, Zhang, and
  Philip]{wu2020comprehensive}
Wu, Z., Pan, S., Chen, F., Long, G., Zhang, C., and Philip, S.~Y.
\newblock A comprehensive survey on graph neural networks.
\newblock \emph{IEEE transactions on neural networks and learning systems},
  32\penalty0 (1):\penalty0 4--24, 2020.

\bibitem[Xia et~al.(2019)Xia, Liu, Wang, Han, Gong, Niu, and
  Sugiyama]{xia2019anchor}
Xia, X., Liu, T., Wang, N., Han, B., Gong, C., Niu, G., and Sugiyama, M.
\newblock Are anchor points really indispensable in label-noise learning?
\newblock \emph{Advances in Neural Information Processing Systems}, 32, 2019.

\bibitem[Xiong et~al.(2022)Xiong, Li, Feng, Deng, Zhang, and
  Hooi]{xiong2022birds}
Xiong, M., Li, S., Feng, W., Deng, A., Zhang, J., and Hooi, B.
\newblock Birds of a feather trust together: Knowing when to trust a classifier
  via adaptive neighborhood aggregation.
\newblock \emph{Transactions on Machine Learning Research}, 2022.
\newblock URL \url{https://openreview.net/forum?id=p5V8P2J61u}.

\bibitem[Yang et~al.(2016)Yang, Cohen, and Salakhudinov]{yang2016revisiting}
Yang, Z., Cohen, W., and Salakhudinov, R.
\newblock Revisiting semi-supervised learning with graph embeddings.
\newblock In \emph{International conference on machine learning}, pp.\  40--48.
  PMLR, 2016.

\bibitem[{Zhu} et~al.(2021){Zhu}, {Dong}, and {Liu}]{zhu2022ICML}
{Zhu}, Z., {Dong}, Z., and {Liu}, Y.
\newblock {Detecting Corrupted Labels Without Training a Model to Predict}.
\newblock \emph{arXiv e-prints}, art. arXiv:2110.06283, October 2021.

\bibitem[Zhu et~al.(2021{\natexlab{a}})Zhu, Liu, and Liu]{zhu2021second}
Zhu, Z., Liu, T., and Liu, Y.
\newblock A second-order approach to learning with instance-dependent label
  noise.
\newblock In \emph{Proceedings of the IEEE/CVF Conference on Computer Vision
  and Pattern Recognition}, pp.\  10113--10123, 2021{\natexlab{a}}.

\bibitem[Zhu et~al.(2021{\natexlab{b}})Zhu, Song, and
  Liu]{zhu2021clusterability}
Zhu, Z., Song, Y., and Liu, Y.
\newblock Clusterability as an alternative to anchor points when learning with
  noisy labels.
\newblock In \emph{International Conference on Machine Learning}, pp.\
  12912--12923. PMLR, 2021{\natexlab{b}}.

\end{thebibliography}
\bibliographystyle{icml2023}

\clearpage
\appendix
\label{sec:app}



\section{Comparison to Learning with Noisy Labels Methods}
\label{app:cmp_learn_with_noise}

Our \method and some learning-with-noise methods all explore neighborhood similarity~\cite{zhu2021clusterability, zhu2021second}. But they focus on i.i.d. image data, while we focus on non-i.i.d. graph data, which differs from i.i.d settings, as our edge relations are 1) of primary importance; and 2) can involve more complex relations other than just similarity. Hence, we presents a flexible approach that learns the patterns of mislabels from data, primarily based on graph structure information.

Synthetic mislabel generation based on transition probabilities is applied in \citet{xia2019anchor, jiang2021information} to help training model with noisy data. But their focus is on how to estimate and predefine a noise transition matrix, while our noise transition matrix is learned and utilized to generate synthetic mislabels in a post-hoc way, which is more flexible and can easily adapt to different data.

\section{Learning Mislabel Transition Matrix}
\label{app:learn_transition}

To generate class-dependent label noise, we learn a \emph{mislabel transition matrix} to capture the probability of mislabelling one class to some other class.


Following the definition in \citet{northcutt2021confident}, we first calculate confident joint $\mathbf{C}_{\tilde{y},y^*}$ based on the model predictions on validation set, which estimates the number of samples with noisy label $\tilde{y}$ and actual true label $y^*$:

\begin{align}
    \mathbf{C}_{\tilde{y},y^*} [i][j] \coloneqq 
    \left| 
    \mathbf{\hat{X}}_{\tilde{y}=i,y^*=j} 
    \right|,
\label{eq:conf-joint}
\end{align}

where $\mathbf{\hat{X}}_{\tilde{y}=i,y^*=j}$ is the estimated set of samples with observed noisy label $i$ and unknown true label $j$. Its formal definition is:

\begin{align}
\begin{split}
    \mathbf{\hat{X}}_{\tilde{y}=i,y^*=j} \coloneqq
    & \{ 
    x \in \mathbf{X}_{\tilde{y}=i} :
    \hat{p} \left( \tilde{y}=j ; x, \theta \right) \ge t_j, \\
    & j = \mathop{\arg\max}\limits_{\mathclap{\scriptscriptstyle l \in [c]: \hat{p} \left( \tilde{y}=j ; x, \theta \right) \ge t_l}}
    \hat{p} \left( \tilde{y}=j ; x, \theta \right) 
    \},
\end{split}
\end{align}

where $\hat{p} \left( \tilde{y}=i ; x, \theta \right)$ is the predicted probability of label $\tilde{y}=i$ for sample $x$ and model parameters $\theta$, and $\mathbf{X}_{\tilde{y}=i}$ is the set of samples with observed noisy label $i$. The threshold $t_j$ is the average self-confidence for class $j$:

\begin{align}
\label{eq:threshold}
    t_j = \frac{1}{\left| \mathbf{X}_{\tilde{y}=j} \right|}
    \mathop{\sum}\limits_{x \in \mathbf{X}_{\tilde{y}=j}} 
    \hat{p} \left( \tilde{y}=j ; x, \theta \right).
\end{align}

As shown in the above formulas, samples in $\mathbf{\hat{X}}_{\tilde{y}=i,y^*=j}$ should satisfy three conditions. First, their observed noisy label $\tilde{y}$ should be $i$. Second, their predicted probability of belonging to class $j$ should be higher than class $j$'s average self-confidence. Third, among all the classes whose predicted probability is larger than their corresponding self-confidence, class $j$'s predicted probability is the highest.

Now we can estimate the joint distribution of noisy label and true label based on the confident joint:

\begin{align}
    \mathbf{\hat{Q}}_{\tilde{y}=i,y^*=j} = 
    \frac
    {\frac{\mathbf{C}_{\tilde{y}=i,y^*=j}}
    {\sum_{j \in [c]} \mathbf{C}_{\tilde{y}=i,y^*=j}}
    \cdot 
    \left| \mathbf{X}_{\tilde{y}=i} \right|}
    {\sum\limits_{i \in [c], j \in [c]} 
    \left(
    \frac{\mathbf{C}_{\tilde{y}=i,y^*=j}}
    {\sum_{j^\prime \in [c]} \mathbf{C}_{\tilde{y}=i,y^*=j^\prime}}
    \cdot 
    \left| \mathbf{X}_{\tilde{y}=i} \right|
    \right)}.
\end{align}

Finally, the mislabel transition matrix can be calculated following the conditional probability formula $\mathbf{\hat{Q}}_{\tilde{y}=i|y^*=j} \coloneqq \mathbf{\hat{Q}}_{\tilde{y}=i,y^*=j} / \mathbf{\hat{Q}}_{y^*=j}$.

\section{Discussion on Sampling Mislabels}
\label{app:discuss_samp_mislabel}

We assume that $\mathcal{V}_{\text{train}}$, $\mathcal{V}_{\text{val}}$, $\mathcal{V}_{\text{test}}$ are all noisy, which means mislabels exist in the raw dataset. 
But according to \citet{northcutt2021confident}, raw mislabels only account for a minor proportion of a dataset. 
Some of the raw mislabels may be chosen as synthesized mislabels. 
There is a slight chance that the labels of these chosen raw mislabels will be flipped to the correct ones. 
Nevertheless, the raw mislabels whose labels remain wrong after \emph{Synthetic Mislabel Dataset Generation} are still a very small portion, which serve as acceptable noises contributing to the robustness of our mislabel detector.

\section{Design of Neighborhood Agreement Features}
\label{app:neighborhood_agreement}

Using both the corrupted and original label matrices ($\mathbf{Y}$ and $\mathbf{Y}_c$) is an important design choice to ensure that \emph{synthetic mislabels only corrupt their own features, not the features of other nodes}. 

Specifically, when we flip the label of a node $v \in \mathcal{V}_{\text{synth}}$, this affects $\mathbf{Y}_c$ but not $\mathbf{Y}$ or $\mathbf{P}$; hence, the features at node $v$ are affected, but those at other nodes are not. This is important as when training the mislabel detector, the fraction of synthetic mislabels is high (due to our use of a balanced training set), so we do not want to allow the mislabels at corrupted nodes to mislead the training of the mislabel detector.


   
   

\section{Time Complexity}
\label{app:complexity}

In this section, we estimate the time complexity of our \method. Given that there are $m$ edges and $s$ epochs, it takes $O(n+c^2)$ to estimate the mislabel transition matrix, $O(mKc)$ to generate the neighborhood agreement features, and $O(nKs)$ to train the mislabel detector. Test time complexity is just $O(K)$ per test sample. Overall, training time is linear and testing time is constant per test sample, which is highly efficient.

\section{Theoretical Guarantees}
\label{app:mislabel-proof}

\paragraph{Overview}
In this section, we show how we can set mislabel score thresholds to obtain guarantees on the false positive and false negative probabilities, using the approach of conformal prediction. Such guarantees can be valuable in many practical settings, where are interested in having guarantees on the reliability of our model's decisions.

Conformal prediction \cite{vovk2005algorithmic,balasubramanian2014conformal} is a simple, distribution-free approach for obtaining confidence guarantees. Importantly, conformal prediction does not require the assumption that the mislabel scores are i.i.d., but only that they are \emph{exchangeable}; i.e., their distribution does not change under any permutation of the sample indices. This is especially suitable for the graph setting, where the samples are in fact not i.i.d.: e.g., samples which are nearby along the graph tend to be correlated. However, the data is still exchangeable, since permutations to the indices do not affect the data distribution.

\paragraph{Comparison to prior work}
Conformal prediction \cite{vovk2005algorithmic} is typically applied to the standard setting where all the samples are exchangeable. Most closely related to our approach is \citet{xiong2022birds}, which extended this to the mislabel detection setting, where a small number of samples can be mislabelled. However, the approach in \citet{xiong2022birds} only allows for practical false positive guarantees, so they do not prove false negative guarantees, while our approach allows for both false positive and negative guarantees due to our use of modified scores, which tightens the bounds by exploiting the accuracy of the mislabel classifier.

Let $\{\cdots\}$ denote a set and $(\cdots)$ denote an ordered tuple: e.g., sorting a set $\{3, 1, 2\}$ yields an ordered tuple $(1,2,3)$.

Before focusing on the mislabel detection setting, we first consider a more general setting where we are given a dataset $\{(\mathbf{x}^{(i)},y^{(i)}) \}_{i=1}^{N}$ with any set of scores $\{s^{(i)} \}_{i=1}^{N}$; we further assume that the dataset comes from a mixture of two distributions: specifically, $N_U$ of the the samples come from some distribution $\mathcal{U}$, i.e., 
$(\mathbf{x},y) \sim \mathcal{U}$, while the remaining $N_V = N - N_U$ samples come from the distribution 
$\mathcal{V}$, i.e., $(\mathbf{x},y) \sim \mathcal{V}$. Let $(s_{(i)})_{i=1}^{N}$ denote the score of these samples in non-decreasing order. 

\begin{theorem}[Conformal Prediction for Mixtures]
    \label{thm:conformal}
    For any given confidence level $\alpha \in (\frac{1}{N+1}, 1)$,
     define the threshold as 
    \begin{align*}
    \lambda_\alpha := s_{(B_\alpha)}, \text{ where } B_\alpha = \left\lceil(N_U+1)(1- \alpha) + N_V \right\rceil.
    \end{align*}

    Now consider a newly drawn sample (e.g., from the test set): 
    \begin{align*}
    (\Tilde{\mathbf{x}}, \Tilde{y}) \sim \mathcal{U}.
    \end{align*}
    
    Then, with probability at least $1-\alpha$ over the random choice of $(\Tilde{\mathbf{x}}, \Tilde{y})$, we have:
    \begin{align*}
    \Tilde{s} \le \lambda_\alpha,
    \end{align*}
    
    where $\Tilde{s}$ is the score of $\Tilde{\mathbf{x}}$.
\end{theorem}

\begin{proof}

Let $(s_{(i)}^{\mathcal{U}})_{i=1}^{N_U}$ denote the scores of the samples from $\mathcal{U}$ in non-decreasing order. The probability that the score of $\Tilde{\mathbf{x}}$ exceeds the threshold $\lambda_\alpha$ is:
\begin{align} 
\label{eq1}
    \mathbb{P}\left(\Tilde{s} > \lambda_\alpha \right) 
    &= \mathbb{P}\left(\Tilde{s} > s_{(B_\alpha)} \right) \\
    &\leq  \mathbb{P}\left(\Tilde{s} > s_{(B_\alpha-N_V)}^{\mathcal{U}} \right) \\ \label{eq:eq_conformal}
    & \leq \frac{N_U+1 - (B_\alpha-N_V)}{N_U+1} \\
    & = \frac{N+1 - B_\alpha}{N_U+1} \\
    &=  \frac{N+1- \left\lceil(N_U+1)(1- \alpha) + N_V \right\rceil }{N_U+1} \\
    &\leq \frac{N+1- (N_U+1)(1- \alpha) - N_V }{N_U+1} \\
    &= \frac{N_U+1- (N_U+1)(1- \alpha)}{N_U+1} \\
    & = \frac{(N_U+1)(\alpha)}{N_U+1} \\
    & = \alpha
\end{align}
Therefore, with probability at least $1-\alpha$, the score $\Tilde{s}$ lies below the threshold, i.e. 
\begin{align}
\label{eq2}
\mathbb{P}\left(\Tilde{s} \leq \lambda_\alpha \right) \geq 1-\alpha.
\end{align}

Note that the step \eqref{eq:eq_conformal} comes from the fact that $\tilde{\mathbf{x}}$ is exchangeable
with the other samples from $\mathcal{U}$~\citep{balasubramanian2014conformal}.
\end{proof}

Next, we show how Theorem \ref{thm:conformal} can be used to obtain both false positive and negative guarantees, for appropriately selected thresholds. 

Given a dataset $\{(\mathbf{x}^{(i)},y^{(i)}) \}_{i=1}^{N}$, and let the mislabel scores computed by our \textsc{GraphCleaner} algorithm be $\{s^{(i)} \}_{i=1}^{N}$. Let $p$ be the fraction of these samples which are mislabelled. 

\begin{proposition}[False Positive Guarantee]
Letting $\lambda_\alpha := s_{(B_\alpha)}$, \text{ where } $B_\alpha = \left\lceil(N(1-p)+1)(1- \alpha) + Np \right\rceil$, with probability at least $1-\alpha$ over the random choice of a new sample $(\Tilde{\mathbf{x}}, \Tilde{y}) \sim \mathcal{U}$, we have:
\begin{align*}
\tilde{s} \le \lambda_\alpha,
\end{align*}
\end{proposition}
\begin{proof}
We apply Theorem \ref{thm:conformal} to the score function $s$, with $\mathcal{U}$ representing the distribution of correctly labelled samples, and $\mathcal{V}$ representing the distribution of mislabelled samples. 
\end{proof}

\paragraph{Discussion} This result allows us to set the threshold $\lambda_\alpha$ in a principled way which provides guarantees on the probability of a false positive (i.e. mistakenly classifying a correctly labelled sample as mislabelled). 

\begin{proposition}[False Negative Guarantee]
Define the modified score function $s' := (1-s)\cdot \mathds{1}_{\{ s > 0.5 \}}$. Then letting $\lambda_\alpha := s'_{(B_\alpha)}$, \text{ where } $B_\alpha = \left\lceil(Np+1)(1- \alpha) + N(1-p) \right\rceil$, with probability at least $1-\alpha$ over the random choice of a new sample $(\Tilde{\mathbf{x}}, \Tilde{y})\sim \mathcal{V}$ and with $\Tilde{s}'$ as its modified score, we have:
\begin{align*}
\Tilde{s}' = (1-\tilde{s})\cdot \mathds{1}_{\{ \tilde{s} > 0.5 \}} \le \lambda_\alpha,
\end{align*}
\end{proposition}
\begin{proof}
In this case, we similarly apply Theorem \ref{thm:conformal} to the modified score function $s'$, but now with $\mathcal{U}$ representing the distribution of mislabelled samples, and $\mathcal{V}$ representing the distribution of correctly labelled samples. 
\end{proof}

\paragraph{Discussion} One difference between $s'$ and $s$ is that the direction of $s'$ is reversed compared to $s$; this is minor and just for ease of interpretation, by making higher scores typically more unlikely.  The main difference between them is that $s'$ is nonzero only when the mislabel classifier predicts the sample to be mislabelled. Using this modified score makes sense since we are primarily interested in setting an appropriate threshold for mislabelled samples, using the distribution of mislabelled samples. Meanwhile, note that it is impractical to achieve false negative bounds using the original scores $s$, due to the large majority of correctly labelled samples which would be expected to have low mislabel scores. In contrast, $s'$ is able to map these samples to $0$ as long as they are correctly classified by the mislabel classifier, allowing a tighter bound.

\paragraph{Remark}
To show the effectiveness of our theoretical guarantees, we calculate the actual false positive rate of our \method's predictions on \texttt{OGB-arxiv} and compare it to the theoretical bounds in Figure \ref{fig:fp-guarantee}.

\begin{figure}[h]
    \centering
    \includegraphics[width=\linewidth]{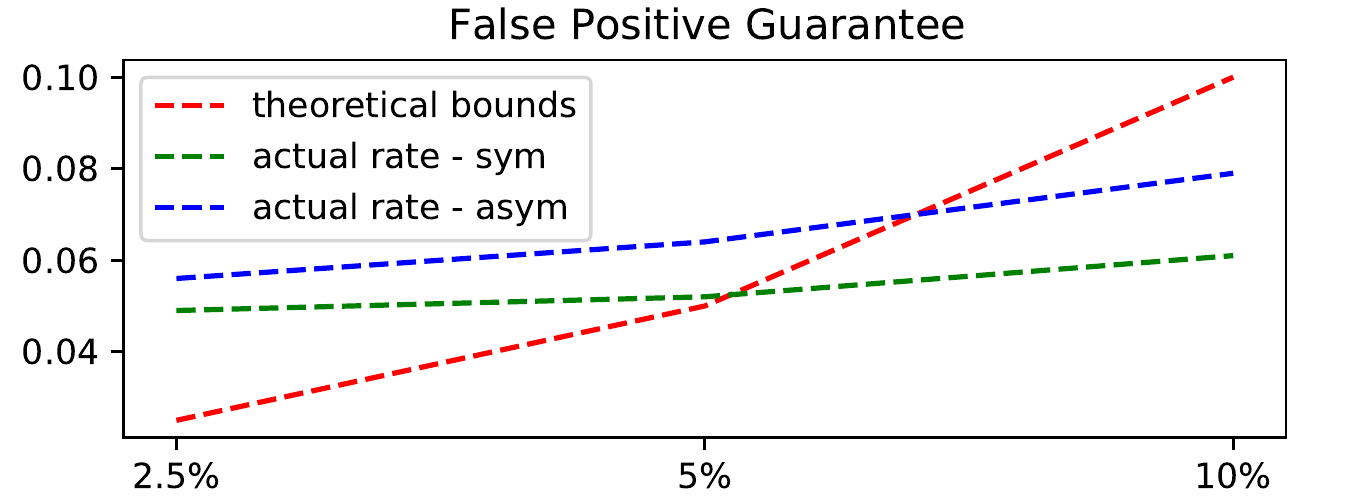}
    \caption{The plot of the theoretical and actual false positive rate under both symmetric and asymmetric noise scenarios. The actual rate is calculated based on the experiments on \texttt{OGB-arxiv}.}
    \label{fig:fp-guarantee}
\end{figure}

\section{Experimental Setup}
\label{app:exp}

\paragraph{Experimental Setup} In order to derive datasets with ground truth labels indicating whether a sample is mislabelled, we randomly introduce artificial noise to $\epsilon$ fraction of the training, validation and test set. We then further assume that the `actual' mislabel ratio of the corrupted dataset is $\epsilon$. 
We follow the practice of INCV~\cite{chen2019understanding} to introduce two mislabelling types in our experiment: symmetric and asymmetric setting, where the probability of changing one class to any other class is equal or different. 
In the asymmetric setting, we simply change class $i$ to class $(i+1) \mod c$. Specifically, if mislabel ratio $\epsilon$ is $0.1$, then for $n_i$ nodes belonging to class $i$, $0.9n_i$ will remain class $i$, while the rest $0.1n_i$ will be mislabelled as class  $(i+1) \mod c$. Illustration for these two mislabelling types is in Figure \ref{fig:mislabel_types}. 

\begin{figure}[htbp]
  \centering
  \begin{minipage}[b]{.46\linewidth}
    \centering
    \includegraphics[width=\linewidth]{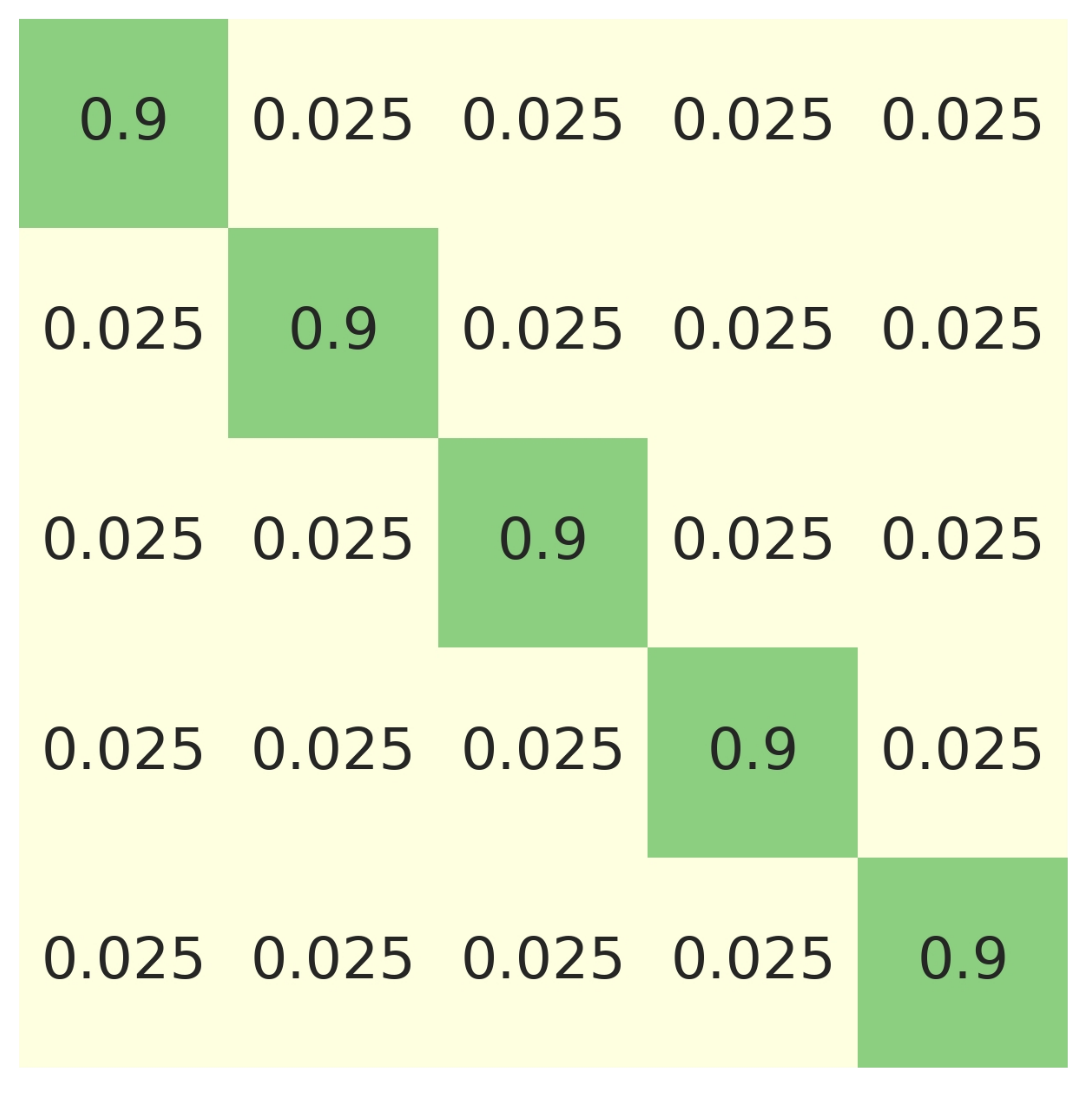}
  \end{minipage}
  \hspace*{1em}
  \begin{minipage}[b]{.46\linewidth}
    \centering
    \includegraphics[width=\linewidth]{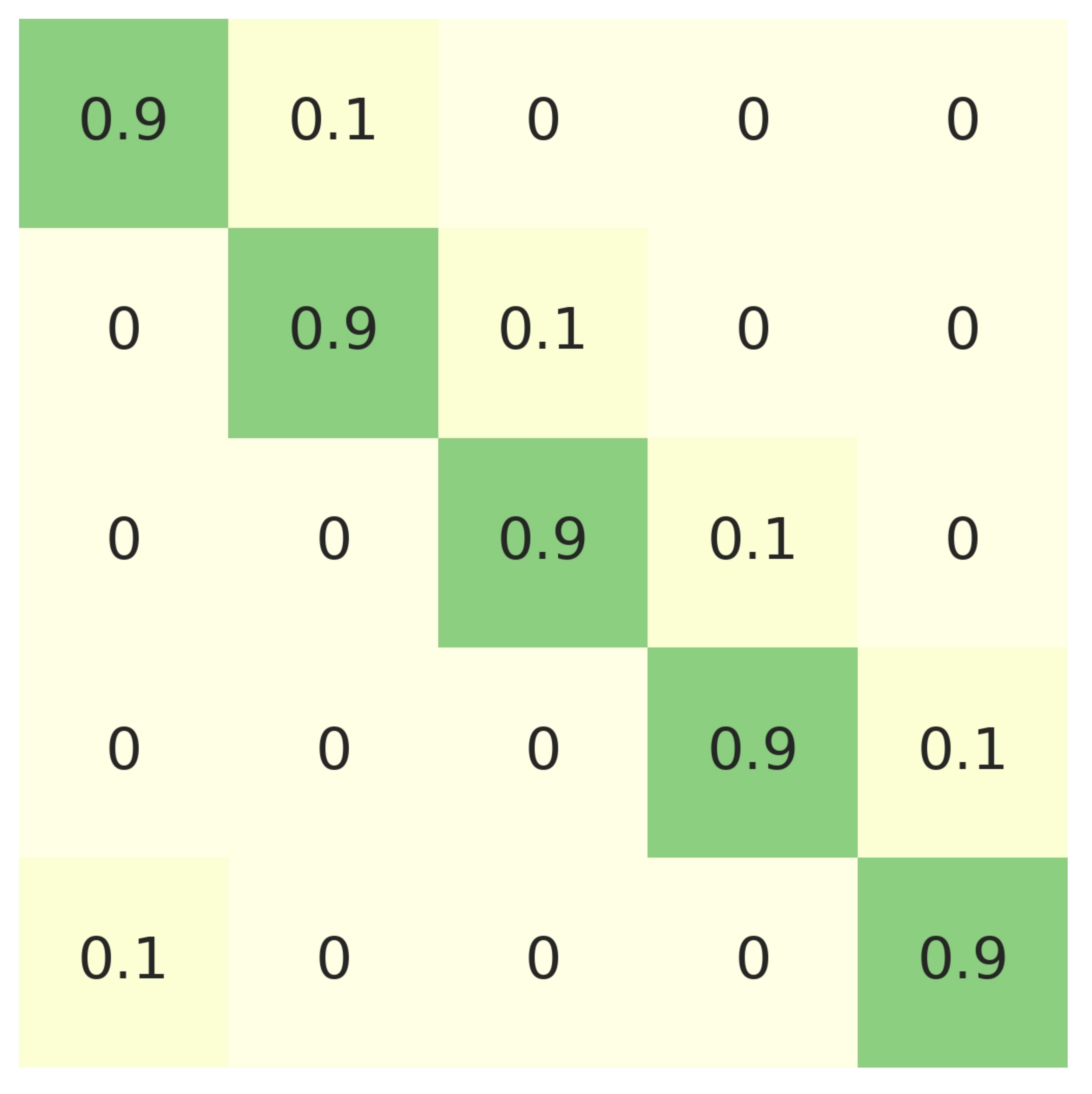}
  \end{minipage}
  \caption{Examples of symmetric (left) and asymmetric (right) mislabelling types (taking 5 classes and mislabel ratio $\epsilon$ 0.1 as an example).}
  \label{fig:mislabel_types}
\end{figure}

We test three mislabel rates $\epsilon$ to show that our \method can tackle various mislabel severities, each with two mislabel types to show that \method consistently performs well under different mislabel patterns. Specifically, $\epsilon$ is set as 0.1, 0.05, 0.025, because the max test set error of different benchmarks reported in \citet{northcutt2021pervasive} is 10.12\%. We do admit that there can be some unknown noisy nodes in $\mathcal{V}_{\text{test}}$. But it is safe to assume that this unknown mislabel rate is minor. In experiments, we compare all methods on the same manually corrupted test set, which is fair. 

\paragraph{Threshold Selection} Our mislabel detector is a binary classifier. Usually, $0.5$ is used as the threshold for binary classifiers' output. But in our case, the training stage is based on a balanced dataset, while the test stage is conducted under an imbalanced setting, which is expected to have only a few percent of mislabelled data. In consideration of this, we adjust the threshold according to Bayes rule. Let $\mathbf{P}_s$ and $\mathbf{P}_t$ denote the binary classifier's predictions on training and test set, then we can have the following equations:
\begin{align}
\label{eq:data-dist}
    \mathbf{P}_s(y|x) = \frac{\mathbf{P}_s(x|y) \mathbf{P}_s(y)}{\mathbf{P}_s(x)},
    \mathbf{P}_t(y|x) = \frac{\mathbf{P}_t(x|y) \mathbf{P}_t(y)}{\mathbf{P}_t(x)}.
\end{align}
We assume that the only change between training and test set is the class (mislabelled or not) distribution, then we have $\mathbf{P}_s(x|y)=\mathbf{P}_t(x|y)$. Dividing the above two equations and plugging $\mathbf{P}_s(x|y)=\mathbf{P}_t(x|y)$ in will yield:
\begin{align}
    \mathbf{P}_t(y|x) = \mathbf{P}_s(y|x) \cdot \frac{\mathbf{P}_t(y)}{\mathbf{P}_s(y)} \cdot \frac{\mathbf{P}_s(x)}{\mathbf{P}_t(x)},
\end{align}
where predictions on test set depend on training set data distribution $\mathbf{P}_s(y|x)$, some class-dependent scaling factor $\frac{\mathbf{P}_t(y)}{\mathbf{P}_s(y)}$, and one constant value $\frac{\mathbf{P}_s(x)}{\mathbf{P}_t(x)}$ which only replies on $x$. This indicates that we can adjust the threshold by $\frac{\mathbf{P}_t(y)}{\mathbf{P}_s(y)}$, an expected value of mislabel proportion. Since the average label error reported in \citet{northcutt2021pervasive} is  $3.4\%$, we simply set the threshold as $0.97$. All our experiments and case studies use this threshold.

\paragraph{Dataset} 6 datasets for node classificatoin are chosen: \texttt{Cora}, \texttt{CiteSeer} and \texttt{PubMed}~\cite{yang2016revisiting}, \texttt{Computers} and \texttt{Photo}~\cite{shchur2018pitfalls}, \texttt{OGB-arxiv}~\cite{hu2020ogb}. 
\texttt{Cora}, \texttt{CiteSeer}, \texttt{PubMed} and \texttt{OGB-arxiv} are citation networks where nodes represent documents and edges represent citation links. \texttt{Computers} and \texttt{Photo} are Amazon co-purchase networks where nodes represent goods and edges represent that two goods are frequently bought together.



\section{Hyperparameters}
\label{app:hyperparameters}

The maximum neighborhood size $K$ determines the range of neighborhood we consider. To investigate the robustness of \method to $K$, we vary $K$ from 1 to 5 with other parameters fixed. 

Figure \ref{fig:ablation_k} shows that our \method is insensitive to the hyperparameter $K$ and even setting $K=1$ yields decent performance, suggesting that using the information of direct neighbors is sufficient in most cases.

\begin{figure}[h]
    \centering
    \includegraphics[width=\linewidth]{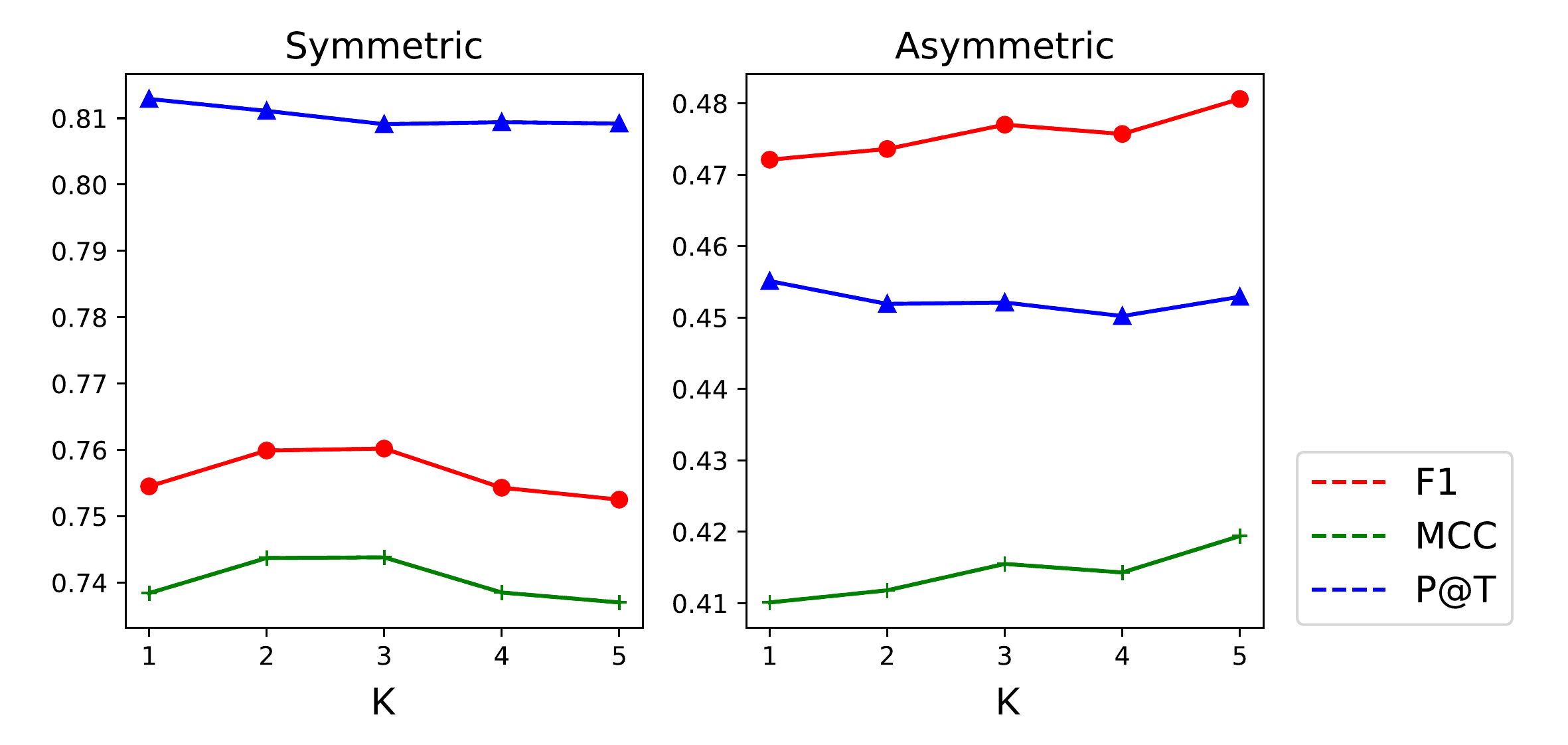}
    \caption{Sensitivity of hyperparameter $K$. Experiments are performed using GCN on \texttt{OGB-arxiv} with 10\% noise.}
    \label{fig:ablation_k}
\end{figure}

\section{Mislabel Case Studies}
\label{app:case_study}

\paragraph{Supplement to Findings}
Different datasets have different levels of label ambiguity: e.g., the classes in \texttt{PubMed} are more objective and precisely defined than those in \texttt{OGB-arxiv}.

Many `ambiguous' and `likely non-mislabel' samples span multiple categories: e.g., a paper predicting financial indicators using methods from natural language processing could reasonably belong to the `finance' or `language' categories.

\paragraph{Limitations} While most samples were straightforward to categorize, we acknowledge that some manual judgments\footnote{For transparency, we include all manual judgments in our attached code repository.} involved unavoidable subjectivity. Even so, our goals in this section are mainly to understand broad overall trends and differences, and our overall findings are relatively robust to small variations. Such subjectivity could be reduced by employing a larger number of manual raters, but this process is fairly labor-intensive. Crowdsourcing could be employed, but is challenging as the task requires some expertise in order to read and categorize papers. Another point is about the inference stage. The focus of our \method is not correcting mislabels. We check the accuracy of the inferred labels of clearly mislabelled samples in \texttt{Cora}, \texttt{CiteSeer} and \texttt{OGB-arxiv}, getting results of $75\%, 100\%$ and $66.67\%$, which is acceptable but still has room for improvement.

\paragraph{Mislabel Rate Estimation on \texttt{PubMed}} 
Can we loosely estimate the total amount of label noise in \texttt{PubMed}? How do these samples affect the use of \texttt{PubMed} for evaluating algorithm performance? We focus on \texttt{PubMed} due to the presence of some auxiliary information that helps us to answer these questions.

Since \texttt{PubMed} contains 19717 samples, estimating the number of mislabels via manual checking is clearly infeasible. Fortunately, it is possible to exactly count the number of samples with label noise of a different kind, as we will explain. Concretely, we use the PubMed API\footnote{https://www.ncbi.nlm.nih.gov/home/develop/api/}, which allows us to query the raw `keywords', known as `MeSH terms'\footnote{MeSH terms are a `vocabulary' or ontology used in PubMed.} associated with each paper. The 3 labels used in the \texttt{PubMed} dataset were assigned based on 3 MeSH terms: `Diabetes Mellitus, Experimental', `Diabetes Mellitus, Type 1', and `Diabetes Mellitus, Type 2'. Thus, querying the PubMed API tells us which papers are assigned with 2 or more of these MeSH terms. For clarity, we refer to these samples as `multi-labelled'. For such cases, it turns out that the \texttt{PubMed} dataset was generated by discarding all but one of such labels for each sample. We note that there is no possibility that the single label kept is somehow indicative of the `most correct' class, since the label to be kept is simply chosen alphabetically. This is a form of label noise or ambiguity, and the number of such samples can be counted exactly via the PubMed API. 

We find that 71 papers are assigned 0 labels (i.e., MeSH terms), 18351 papers are assigned 1 label, 1256 papers are assigned 2 labels, and 39 papers are assigned 3 labels. The latter two categories are multi-labels, while the 0-label papers occur due to updates in the MeSH terms assigned to some papers over time. In summary, we have 1366 papers with some form of label noise, or $6.91\%$ of the entire dataset. Note that this is only a lower bound for the true amount of label noise, as it does not include mislabels such as those in Figure \ref{fig:mislabels}.

\paragraph{Effect of Correcting Label Noise} What are the implications of these $6.91\%$ label noise samples? The top reported leaderboard scores on \texttt{PubMed}\footnote{https://paperswithcode.com/dataset/pubmed} are close to $90\%$ accuracy, suggesting that a significant fraction of the remaining `missing accuracy' could be attributed to label noise. This supports a `data-centric' view, which recognizes data quality (e.g., label noise) as an important factor affecting performance and evaluation, and suggests that correcting mislabels has scope for significant value.

Finally, we study the effect of correcting label noise, by comparing the accuracy of a simple GCN+mixup baseline before and after removing the above-mentioned noisy-labelled samples in the test set. The accuracy improves from $86.71\%$ to $89.11\%$, suggesting that label noise has significant effects on performance evaluation. Thus, recognizing the importance of correcting this label noise, we publicly release two new variants of the \texttt{PubMed} dataset: 1) \texttt{PubMedCleaned}, which removes these noisy-labelled samples, and also corrects all mislabels we have detected; 2) \texttt{PubMedMulti}, which keeps multi-labelled samples but explicitly assigns them multiple labels, for users to develop algorithms which can handle the multi-labelling scenario. Please refer to our code link for these two new datasets and the manual judgement files for case studies.

\end{document}